\documentclass[sigconf,nonacm]{aamas}

\usepackage{balance} 
\usepackage{multirow}
\usepackage{array}
\usepackage{xspace}
\usepackage{tikz}
\usetikzlibrary{positioning,shapes.geometric,arrows.meta,shadows}

\setcopyright{none}
\AtBeginDocument{%
  \fancypagestyle{standardpagestyle}{\fancyhf{}}%
  \fancypagestyle{firstpagestyle}{\fancyhf{}}%
  \pagestyle{standardpagestyle}}
\AtEndPreamble{\excludecomment{acks}}

\title[Testing the Planning Claims of LLM-Based Agents]{Testing the Planning Claims of LLM-Based Agents against~Formal~Semantics}

\makeatletter
\newcommand{\numberedauthor}[3]{%
  \parbox[t]{0.315\textwidth}{\centering
    {\sffamily\bfseries\normalsize #1\textsuperscript{#2}\par}%
    \vspace{2pt}%
    {\sffamily\scriptsize\href{mailto:#3}{\mbox{#3}}\par}}}
\newcommand{\numberedaffiliation}[2]{%
  \parbox[t]{0.47\textwidth}{\raggedright\sffamily\footnotesize
    \textsuperscript{#1}\,#2\par}}
\def\@mkauthors@iii{%
  \global\setbox\mktitle@bx=\vbox{%
    \unvbox\mktitle@bx\par\vskip 1em
    \hbox to \textwidth{%
      \numberedauthor{Oriol Miro-Lopez-Feliu}{1}{oriol.miro.07@gmail.com}\hfil
      \numberedauthor{Oscar Miró López-Feliu}{2,3}{oscar.miro.lopez.feliu@student.uva.nl}\hfil
      \numberedauthor{Victor Gimenez-Abalos}{1,4}{victor.gimenez@bsc.es}%
    }%
    \vskip 1.2em
    \hbox to \textwidth{%
      \numberedauthor{Javier Vazquez-Salceda}{4}{jvazquez@cs.upc.edu}\hfil
      \numberedauthor{Victor Carballo Araruna}{1}{victor.carballo@bsc.es}\hfil
      \numberedauthor{Sergio Alvarez-Napagao}{1,4}{sergio.alvarez@bsc.es}%
    }%
    \vskip 1.2em
    \hrule height 0.25pt
    \vskip 0.6em
    \hbox to \textwidth{%
      \numberedaffiliation{1}{Barcelona Supercomputing Center, Barcelona, Spain}\hfil
      \numberedaffiliation{2}{University of Amsterdam, Amsterdam, The Netherlands}%
    }%
    \vskip 0.45em
    \hbox to \textwidth{%
      \numberedaffiliation{3}{Safe AI Netherlands, The Netherlands}\hfil
      \numberedaffiliation{4}{Universitat Politècnica de Catalunya, Barcelona, Spain}%
    }%
    \par\vskip 1.2em
  }%
}
\makeatother

\author{Oriol Miro-Lopez-Feliu}
\affiliation{
  \institution{Barcelona Supercomputing Center}
  \city{Barcelona}
  \country{Spain}}
\email{oriol.miro.07@gmail.com}

\author{Oscar Miró López-Feliu}
\affiliation{
  \institution{University of Amsterdam}
  \city{Amsterdam}
  \country{The Netherlands}}
\affiliation{
  \institution{Safe AI Netherlands}
  \country{The Netherlands}}
\email{oscar.miro.lopez.feliu@student.uva.nl}

\author{Victor Gimenez-Abalos}
\affiliation{
  \institution{Barcelona Supercomputing Center}
  \institution{Universitat Politècnica de Catalunya}
  \city{Barcelona}
  \country{Spain}}
\email{victor.gimenez@bsc.es}

\author{Javier Vazquez-Salceda}
\affiliation{
  \institution{Universitat Politècnica de Catalunya}
  \city{Barcelona}
  \country{Spain}}
\email{jvazquez@cs.upc.edu}

\author{Victor Carballo Araruna}
\affiliation{
  \institution{Barcelona Supercomputing Center}
  \city{Barcelona}
  \country{Spain}}
\email{victor.carballo@bsc.es}

\author{Sergio Alvarez-Napagao}
\affiliation{
  \institution{Barcelona Supercomputing Center}
  \institution{Universitat Politècnica de Catalunya}
  \city{Barcelona}
  \country{Spain}}
\email{sergio.alvarez@bsc.es}

\newif\ifdraft
\draftfalse

\newcommand{\UNSAT}{\texttt{UNSOLVABLE}\xspace}
\newcommand{\dom}[1]{\textsc{#1}\xspace}
\newcommand{\tool}[1]{\texttt{#1}}
\newcommand{\Lstar}{L^{\star}}
\newcommand{\etal}{et~al.\xspace}
\newcommand{\eg}{e.g.,\xspace}

\usepackage{listings}
\usepackage[most]{tcolorbox}
\usepackage{colortbl}

\definecolor{accent}{HTML}{0072B2}
\definecolor{accentdark}{HTML}{00466E}
\definecolor{sky}{HTML}{56B4E9}
\definecolor{epistemic}{HTML}{0072B2}
\definecolor{executed}{HTML}{009E73}
\definecolor{illformed}{HTML}{E69F00}
\definecolor{fatal}{HTML}{D55E00}
\definecolor{terminal}{HTML}{CC79A7}
\definecolor{vermilion}{HTML}{D55E00}
\definecolor{placeholder}{HTML}{A0467A}
\definecolor{offlinetint}{HTML}{EAF3FA}
\definecolor{onlinetint}{HTML}{FBEDE5}
\definecolor{breakgrey}{gray}{0.6}
\definecolor{commentgrey}{gray}{0.5}

\newcommand{\minihead}[1]{\par\medskip{\sffamily\bfseries\color{accentdark}#1}\par\nopagebreak\vspace{1pt}}

\lstdefinestyle{base}{basicstyle=\ttfamily\footnotesize, columns=fullflexible, keepspaces=true, aboveskip=0pt, belowskip=0pt,
  breaklines=true, breakatwhitespace=false, showstringspaces=false, upquote=true, extendedchars=true,
  postbreak=\mbox{\textcolor{breakgrey}{$\hookrightarrow$}\space},
  literate={—}{{---}}1 {–}{{--}}1 {’}{{'}}1 {‘}{{`}}1 {“}{{``}}1 {”}{{''}}1}
\lstdefinelanguage{PDDL}{sensitive=false, alsoletter={:-?}, morecomment=[l]{;}, keywordsprefix=?, morekeywords=[1]{nokeyword},
  morekeywords=[2]{define,domain,problem,and,not,or,increase,decrease,forall,exists,when,either},
  morekeywords=[3]{:requirements,:types,:predicates,:functions,:action,:parameters,:precondition,:effect,
    :objects,:init,:goal,:metric,:domain,minimize,:strips,:typing,:negative-preconditions,:equality,
    :fluents,:disjunctive-preconditions}}
\lstdefinestyle{pddl}{style=base, language=PDDL, keywordstyle=[1]\color{vermilion},
  keywordstyle=[2]\color{accent}\bfseries, keywordstyle=[3]\color{accentdark}\bfseries, commentstyle=\color{commentgrey}}
\lstdefinestyle{brief}{style=base, moredelim=[l][\color{accentdark}\bfseries]{\#}}
\lstdefinestyle{prompt}{style=base, moredelim=[s][\color{placeholder}\bfseries]{<}{>}}
\lstdefinestyle{dspy}{style=base, moredelim=[s][\color{accent}\bfseries]{[[}{]]}}
\lstdefinestyle{obs}{style=base}

\tcbset{appbox/.style={enhanced, breakable, boxrule=0.6pt, arc=1.5pt, left=5pt, right=5pt, top=3pt, bottom=3pt,
    before skip=5pt, after skip=8pt, fonttitle=\sffamily\bfseries\small, toptitle=1.5pt, bottomtitle=1.5pt},
  epistemic/.style={colframe=epistemic, colback=epistemic!3},
  executed/.style={colframe=executed, colback=executed!4},
  illformed/.style={colframe=illformed, colback=illformed!5},
  fatal/.style={colframe=fatal, colback=fatal!4},
  terminal/.style={colframe=terminal, colback=terminal!5}}
\newtcbinputlisting{\pddlfile}[2]{appbox, colframe=accent, colback=accent!2, title={#1},
  listing file={#2}, listing only, listing options={style=pddl}}
\newtcbinputlisting{\brieffile}[2]{appbox, colframe=sky, colback=sky!7, coltitle=black, title={#1},
  listing file={#2}, listing only, listing options={style=brief}}
\newtcbinputlisting{\plainfile}[2]{appbox, colframe=black!55, colback=black!2, title={#1},
  listing file={#2}, listing only, listing options={style=obs}}
\newtcbinputlisting{\toolcall}[3]{appbox, #1, title={\ttfamily #2}, listing file={#3}, listing only,
  listing options={style=obs}}
\newtcbinputlisting{\promptfile}[3]{appbox, colframe=black!65, colback=black!2,
  title={#1\hfill{\normalfont\footnotesize #2}}, listing file={#3}, listing only, listing options={style=prompt}}
\newtcbinputlisting{\msgfile}[2]{appbox, colframe=accentdark, colback=accent!2, title={#1},
  listing file={#2}, listing only, listing options={style=dspy}}
\newcommand{\stepbox}[6]{%
  \begin{tcolorbox}[appbox, #1, title={Step #2\quad\texttt{#3}}]
  {\small\itshape\color{black!70}#4}\par\smallskip
  {\small\ttfamily\bfseries\color{#1}#5}\par\smallskip
  \lstinputlisting[style=obs]{#6}
  \end{tcolorbox}}
\newcommand{\legendchip}[2]{\tcbox[on line, colframe=#1, colback=#1!8, boxrule=0.5pt, arc=1pt, left=2pt, right=2pt,
  top=0.5pt, bottom=0.5pt, boxsep=0pt]{\sffamily\scriptsize #2}}

\begin{abstract}
LLM-based agents are increasingly credited with the competences of practical reasoning: planning, monitoring, and correcting their own actions. Their success is attributed to the prompting strategies and agent architectures built around the model, such as chain-of-thought, search-based prompting, self-refinement, reflection, critics and tool-using loops. Yet these claims are usually supported by end-to-end success rates, which show only whether a task was completed, or by LLM judges, which are not verifiers.
We test LLM planning against formal semantics. Tasks are posed in natural language, but each is grounded in a hidden PDDL model that decides whether every action is applicable, whether the goal is reached, and whether a plan exists at all.
Across five domains (two of them new, all released with the paper), four models, six prompting strategies and three agent architectures, the best configuration solves 78\% of the solvable tasks offline but only 40\% online, where agents mostly act without inspecting the state.
Unsolvable tasks are recognised in under 10\% of episodes, and heavier prompting strategies use 5 to 22 times as many tokens as direct prompting with no clear gain.
In our experiments, planning performance varied more across models than across prompting strategies or agent architectures.
Claims that agents can plan require evidence from validated action traces.
\end{abstract}


\begin{document}

\maketitle

\section{Introduction}
\label{sec:intro}

Large language models (LLMs) are increasingly deployed as the core of \emph{LLM agents} in real applications. Without changing the model's parameters, developers embed it in two kinds of structure. \emph{Prompting strategies} elicit intermediate reasoning~\cite{wei2022cot}, critique and revise drafts, or search over candidate continuations; \emph{agent architectures} run control loops that call tools and act on an environment. In the vocabulary of practical reasoning~\cite{bratman1987intention,wooldridge2009introduction}, these strategies and architectures are credited with deliberation, means--ends reasoning, and the monitoring and repair of the agent's own course of action. In ReAct, reasoning traces ``help the model induce, track, and update action plans''~\cite{yao2023react}. Tree of Thoughts is described as enabling ``deliberate decision making'' on planning and search tasks~\cite{yao2023tot}. Self-refinement, reflection and tool-based critics are presented as forms of self-monitoring and correction~\cite{madaan2023selfrefine,shinn2023reflexion,gou2024critic}. LLM agent architectures for planning and agentic workflows are now also studied within the agents community~\cite{choi2026reactree,hsiao2026procedural}, alongside calls to ground agentic AI in established agent concepts~\cite{dignum2025agentifying}.

Such claims concern properties of behaviour: that an agent respects constraints at every step, reaches its goals over long horizons, corrects its own errors, and recognises when a goal cannot be achieved. The evidence offered for them is usually an end-to-end success rate~\cite{liu2024agentbench,zhou2024webarena,mialon2024gaia}, a judgement by another LLM~\cite{zheng2023judging}, or the persuasiveness of a generated rationale, which need not reflect the computation behind the answer~\cite{turpin2023unfaithful}. None of these checks the properties being claimed. A trajectory may reach the goal after an illegal intermediate action. A run may fail because of a single violated precondition or because it never approached the goal. An agent may exhaust its budget pursuing a goal that was unachievable from the outset. Outcome-only evaluation conflates these cases, and surveys and benchmark studies of agent evaluation note the lack of fine-grained, diagnostic signals~\cite{yehudai2025surveyagents,ma2024agentboard}.

Automated planning provides the formal criteria for testing these claims. A plan is valid only if every action is applicable in the state in which it is executed and the final state satisfies the goal~\cite{ghallab2004automated,geffner2013concise}. Validity can be checked mechanically~\cite{howey2004val}, and unsolvability can be proved and certified~\cite{eriksson2017unsolvability}. Agent theory contributes a further distinction, which current usage of ``agentic'' blurs. An agent is characterised by the mapping from its history of interaction with an environment to its next action~\cite{russell2021aima,wooldridge2009introduction}. Many systems described as agentic never extend that history: they reason internally and then emit an answer. We therefore use \emph{LLM-based system} as the umbrella term and reserve \emph{agent} for systems that act on an environment, whether by sensing it or by changing it. \emph{Planning-by-acting}, in which a controller interleaves observation, planning and execution~\cite{ghallab2016acting}, constitutes a different regime. Within it, actions that inspect the world (epistemic actions) differ from actions that change it (ontic actions)~\cite{bolander2011epistemic}.

We test these planning claims against formal semantics. Each task is posed in natural language but grounded in a hidden PDDL model~\cite{mcdermott1998pddl}. In the \emph{offline} regime, a system receives the task description and must return a complete plan or declare the goal unachievable. In the \emph{online} regime, an agent must achieve the goal through a tool interface that separates epistemic from ontic actions, and an inapplicable action ends the episode as a safety violation (Figure~\ref{fig:protocols}). Since an online agent can always retrieve the task description, plan as an offline system would and then execute the plan, the online regime is never harder than the offline one, given sufficient budget (Proposition~\ref{prop:dominance}). Any shortfall of an online agent relative to offline planning with the same model therefore lies in the agent rather than in the environment.

Every plan and every executed action is validated against the hidden model. This allows us to measure executability, goal achievement, efficiency, infeasibility recognition and sensing behaviour, to locate the first failure in each trace, and to relate these outcomes to token usage and monetary cost. Validating LLM-generated plans against PDDL is an established practice~\cite{valmeekam2023planbench,stein2025autoplanbench}. What is new here is its use to test the competences claimed for prompting strategies and agent architectures, with unsolvable tasks and an online regime that separates sensing from acting.

We compare six prompting strategies and three agent architectures across four models from two families, in 1,800 validated episodes. The results provide little evidence that prompting strategies or agent architectures give LLM-based systems planning competence. Offline, the best configuration solves 78\% of the solvable tasks. Gains from prompting are model-specific: chain-of-thought and Plan-and-Solve raise success by up to 20 points for three of the four models, but no strategy improves on direct prompting for Gemini~2.0 Flash-Lite. No gain survives correction for multiple comparisons. Tree of Thoughts consumes an order of magnitude more tokens than direct prompting and is Pareto-dominated. Online, success declines for every model, to at most 40\%. Agents change the state far more often than they inspect it, and precondition violations terminate most failed episodes; neither reflective memory nor a critic improves on ReAct. In both regimes, systems rarely recognise that a goal is unachievable, yet they almost never abstain on solvable tasks. Overall, success varies more across models than across prompting strategies or agent architectures.

\paragraph{Contributions.}
\begin{enumerate}
    \item A formal evaluation framework that operationalises the planning claims made for LLM-based agents as properties checked against an explicit transition model, separates offline planning from online planning-by-acting through the interaction history, and shows that the online regime weakly dominates the offline one (Section~\ref{sec:framework}).
    \item An open benchmark and harness: 50 natural-language tasks backed by hidden PDDL models, in five domains (two of them new) and including ten unsolvable tasks; a validation oracle; an agent environment with epistemic and ontic actions under fatal-violation semantics, which logs every trace with cost accounting (Section~\ref{sec:benchmark}).
    \item A controlled study of six prompting strategies and three agent architectures under identical semantics, with failure-mode, sensing and cost--performance analyses (Sections~\ref{sec:setup} and~\ref{sec:results}).
\end{enumerate}

\begin{figure*}[t]
\centering
\begin{minipage}[b]{0.36\textwidth}
\centering
\resizebox{120pt}{!}{\begin{tikzpicture}[
    font=\small,
    node distance=0.55cm,
    data/.style={
        rectangle,
        draw=gray!45,
        fill=white,
        dashed,
        minimum width=3.85cm,
        minimum height=0.72cm,
        align=center,
        inner sep=2pt
    },
    block/.style={
        rectangle,
        draw=blue!65,
        fill=blue!5,
        thick,
        rounded corners=3pt,
        minimum width=3.85cm,
        minimum height=0.82cm,
        align=center,
        inner sep=2pt
    },
    decision/.style={
        diamond,
        draw=orange!75,
        fill=orange!8,
        thick,
        aspect=2.05,
        minimum width=2.85cm,
        minimum height=0.72cm,
        inner sep=1pt,
        align=center
    },
    outcome/.style={
        rectangle,
        rounded corners=3pt,
        draw=gray!40,
        fill=white,
        minimum width=1.55cm,
        minimum height=0.54cm,
        align=center,
        font=\bfseries\scriptsize,
        inner sep=2pt
    },
    line/.style={draw, -latex, thick, gray!78}
]

    \node[data] (input) {Briefs $B(\Pi)$};
    \node[block, below=of input] (llm) {\textbf{LLM-Based System}};
    \node[data, below=of llm] (plan) {Full Plan $\pi$ or \texttt{UNSOLVABLE}};
    \node[decision, below=0.57cm of plan] (eval) {Evaluator};

    \node[outcome, below left=0.62cm and 0.30cm of eval, text=green!55!black] (success) {Success};
    \node[outcome, below right=0.62cm and 0.30cm of eval, text=red!65!black] (fail) {Fail / Invalid};

    \path[line] (input) -- (llm);
    \path[line] (llm) -- (plan);
    \path[line] (plan) -- (eval);
    \path[line] (eval) -- node[left, font=\footnotesize, text=black, align=right, xshift=-3pt, yshift=3pt] {valid or\\correct $\bot$} (success);
    \path[line] (eval) -- node[right, font=\footnotesize, text=black, xshift=3pt, yshift=3pt] {Error} (fail);

\end{tikzpicture}}

\vspace{2pt}
{\small (a) Offline planning}
\end{minipage}%
\begin{minipage}[b]{0.56\textwidth}
\centering
\resizebox{236pt}{!}{\begin{tikzpicture}[
    font=\normalsize,
    >=latex,
    thick,
    agent/.style={
        rectangle,
        draw=blue!65,
        fill=blue!5,
        thick,
        minimum width=1.58cm,
        minimum height=0.84cm,
        rounded corners=3pt,
        font=\bfseries\small,
        align=center,
        inner sep=2pt
    },
    env_container/.style={
        rectangle,
        draw=green!55!black,
        fill=green!5,
        thick,
        minimum width=5cm,
        minimum height=3.05cm,
        rounded corners=5pt
    },
    validator/.style={
        diamond,
        draw=orange!75,
        fill=orange!10,
        thick,
        aspect=2.0,
        minimum width=1.58cm,
        minimum height=0.56cm,
        inner sep=1pt,
        font=\footnotesize\bfseries,
        align=center
    },
    state_node/.style={
        rectangle,
        draw=green!70!black,
        fill=white,
        dashed,
        thick,
        minimum width=1.65cm,
        minimum height=0.80cm,
        inner sep=2pt,
        font=\footnotesize,
        align=center
    },
    term/.style={
        rectangle,
        rounded corners=2pt,
        draw=gray!45,
        fill=white,
        minimum width=1.10cm,
        minimum height=0.55cm,
        font=\scriptsize,
        align=center,
        inner sep=2pt
    },
    path_ontic/.style={draw=orange!80!black, ->},
    path_epistemic/.style={draw=blue!80!black, ->},
    path_valid/.style={draw=green!60!black, ->},
    path_obs/.style={draw=gray!80, ->},
    path_error/.style={draw=red!75!black, ->, dashed},
    path_goal/.style={draw=blue!60!black, ->, dashed}
]

    \node[agent] (ag) {LLM-Based\\Agent};

    \node[env_container, right=0.50cm of ag] (env) {};
    \node[anchor=north west, font=\bfseries\scriptsize, text=green!38!black, inner sep=4pt]
        at (env.north west) {Environment};

    \node[validator] (val) at ([xshift=-0.50cm,yshift=0.70cm]env.center) {Validator};
    \node[state_node] (state) at ([xshift=-0.50cm,yshift=-0.76cm]env.center) {State /\\Simulator\\$s$};
    \node[term, text=red!70!black] (err) at ([xshift=1.35cm,yshift=0.70cm]env.center) {Precond.\\violation};
    \node[term] (goal) at ([xshift=1.35cm,yshift=-0.76cm]env.center) {Submit\\goal check};

    \coordinate (entry) at ([xshift=-0.08cm]env.west);
    \coordinate (split) at ([xshift=0.34cm]entry);

    \draw[draw=gray!80, ->] (ag.east) -- node[above, pos=0.50, font=\scriptsize] {$x_t$} (entry);
    \draw[path_ontic] (entry) -- (split) |- node[near end, above, font=\tiny] {Ontic} (val.west);
    \draw[path_epistemic] (entry) -- (split) |- node[near end, above, font=\tiny] {Epistemic} (state.west);

    \draw[path_error] (val.east) -- (err.west);
    \draw[path_valid] (val.south) -- node[right, font=\tiny] {Valid} (state.north);
    \draw[path_goal] (state.east) -- (goal.west);

    \draw[path_obs] (state.south)
        -- ++(0,-0.46) coordinate(obsbend)
        -| ([yshift=-0.46cm]ag.south)
        -- (ag.south);
    \node[anchor=north, font=\scriptsize] at ([xshift=-0.62cm]state.south |- obsbend) {Observation $o_t$};

\end{tikzpicture}%
}

\vspace{2pt}
{\small (b) Online planning-by-acting}
\end{minipage}
\caption{Evaluation protocols. (a) Offline, the planner maps the natural-language task to a complete plan or \UNSAT, which is validated once. (b) Online, the agent interacts through tools: epistemic calls return observations without changing the state, ontic calls are validated before the state changes, and an inapplicable action ends the episode.}
\label{fig:protocols}
\Description{Two diagrams side by side. The first shows a pipeline from the natural-language briefs to an LLM-based system, which outputs a full plan or UNSOLVABLE; an evaluator then marks it as success or failure. The second shows an LLM-based agent issuing calls to an environment. Epistemic actions go directly to the state and return observations. Ontic actions go through a validator: valid ones update the state, invalid ones cause a precondition violation. A submit action triggers the goal check.}
\end{figure*}
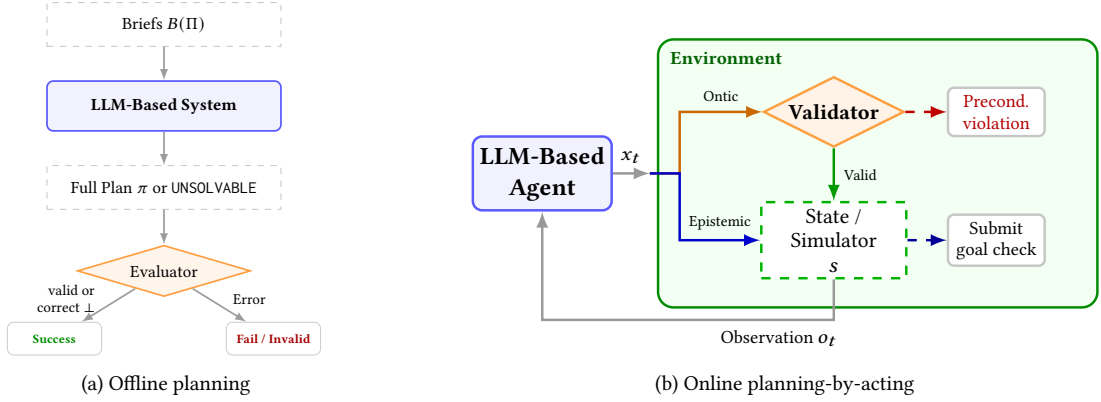

\section{Related Work}
\label{sec:related}

Much of the evidence on whether LLMs can plan comes from classical planning, where every plan can be validated. PlanBench~\cite{valmeekam2023planbench} and the accompanying critical investigation~\cite{valmeekam2023planning} established that LLM-generated plans often look plausible yet fail validation, especially once action and object names are obfuscated. Follow-up work found that chain-of-thought gains in Blocksworld depend on narrow, problem-specific prompts~\cite{stechly2024thoughtlessness}, that LLM self-critique degrades plans while sound external verifiers improve them~\cite{stechly2025selfverification}, and that reasoning models improve markedly but remain unreliable on larger and obfuscated problems~\cite{valmeekam2024lrm}. These findings motivated LLM-Modulo architectures, in which sound external verifiers carry correctness~\cite{kambhampati2024llmmodulo}, and hybrids in which LLMs write models or programs for planners~\cite{liu2023llmp,guan2023worldmodels,silver2024generalized}. We study the complementary setting of general-purpose LLM-based systems with no planner in the loop, which is where the planning claims above are made.

Subsequent benchmarks keep this formal grounding but change what is measured. ACPBench~\cite{kokel2025acpbench} probes atomic reasoning about applicability, progression and reachability rather than plan generation. Closest to our setting, AutoPlanBench~\cite{stein2025autoplanbench} translates PDDL domains into natural language with an LLM, as we do, and compares non-interactive planning with interactive ReAct-style action choice, while G\"obel \etal~\cite{gobel2026stepwise} expose a PDDL simulator to an LLM through tool calls on Blocksworld. We share their use of formal validation but ask a different question. In our setting, agents must choose between inspecting the state and changing it, and an inapplicable action is fatal, so that sensing and unsafe commitment become measurable. Unsolvable tasks further test whether systems recognise that a goal is unachievable, and Proposition~\ref{prop:dominance} attributes any online shortfall relative to offline planning to the agent rather than to the environment. Finally, we compare six prompting strategies and three agent architectures under the same semantics, with cost accounting.

In contrast, interactive benchmarks place LLM agents in web, operating-system, software and household environments~\cite{liu2024agentbench,zhou2024webarena,xie2024osworld,jimenez2024swebench,shridhar2021alfworld,mialon2024gaia}. Their realism comes with implicit dynamics: success is checked by scripts or end-state predicates, so failures may stem from perception, interface brittleness or missing knowledge as much as from planning, and success does not certify that intermediate actions were admissible. Attempts to sharpen the evaluation signal leave this gap open. AgentBoard~\cite{ma2024agentboard} adds a progress-rate metric to expose partial progress. APB~\cite{sun2026apb} provides a large diagnostic suite that separates planning failures from execution failures and includes unsolvable cases, but it has no transition model against which each action can be validated. LLM judges, a common fallback, are sensitive to framing and prone to leniency~\cite{zheng2023judging,thakur2025judging}. Our benchmark trades realism for an explicit transition model, which makes every action checkable.

An explicit transition model also lets us examine whether prompting strategies and agent architectures deliver what is claimed for them. Kapoor \etal~\cite{kapoor2025agents} observe that complex agents are often dominated by simple baselines and argue that agent evaluations must therefore report cost alongside accuracy; we follow that recommendation. Studies of individual strategies raise similar doubts. Intrinsic self-correction often fails to improve reasoning and can degrade it~\cite{huang2024selfcorrect}. The performance of ReAct owes more to the similarity between exemplars and queries than to interleaved reasoning~\cite{verma2024brittle}. LLM-in-the-loop search such as Tree of Thoughts gives up soundness and completeness at a high cost in model calls~\cite{katz2024thought}. These studies, however, typically examine a single strategy in a single regime. We compare prompting strategies and agent architectures side by side, on the same tasks and with the same measurements, which separates their effects on executability, goal achievement and abstention.

The agents literature offers precise notions of agency~\cite{wooldridge1995intelligent,franklin1997agent}, architectures for practical reasoning~\cite{rao1995bdi}, and work on integrating planning into agent reasoning~\cite{meneguzzi2015planning}. It also studies how planning, execution and monitoring are interleaved~\cite{brenner2009continual,ghallab2016acting} and how often an agent should reconsider its commitments~\cite{kinny1991commitment}. Recent work connects these concepts to LLMs, for instance by using language models for plan selection in BDI agents~\cite{ichida2024natbdi}; Dignum and Dignum argue that agentic AI should build on the conceptual tools of the AAMAS community~\cite{dignum2025agentifying}, and Wooldridge revisits multi-agent systems research in light of LLMs~\cite{wooldridge2026rethinking}. We use two of these concepts for evaluation. The agent function over interaction histories separates offline planning from planning-by-acting, and the distinction between epistemic (sensing) and ontic actions~\cite{scherl2003knowledge,bolander2011epistemic} allows us to measure whether agents sense before they act.

Dropping an intention that has become unachievable~\cite{cohen1990intention} has a counterpart in planning, where recognising infeasibility is a classical problem with formal certificates~\cite{eriksson2017unsolvability}. LLMs handle this poorly: abstention on unanswerable questions remains weak and has been reported to degrade with reasoning fine-tuning~\cite{kirichenko2025abstentionbench}, and in planning, o1-preview recognised only 27\% of unsolvable Blocksworld instances~\cite{valmeekam2024lrm}. Plancraft~\cite{dagan2025plancraft} adds intentionally unsolvable tasks to an interactive crafting benchmark for LLM agents. We test abstention in every domain and in both regimes.
\section{Evaluation Framework}
\label{sec:framework}

We treat planning as a set of behavioural properties defined by an explicit transition model. Systems see tasks only in natural language, whereas correctness is decided by hidden formal semantics.

\subsection{Tasks behind natural-language briefs}
\label{sec:tasks}

We use the deterministic model of classical planning~\cite{ghallab2004automated}.

\begin{definition}[Planning task]
\label{def:task}
A \emph{planning task} is a tuple $\Pi=\langle S,A,\gamma,s_0,G\rangle$ of states $S$, grounded actions $A$, a partial transition function $\gamma: S\times A \rightharpoonup S$, an initial state $s_0\in S$ and goal states $G\subseteq S$ with $s_0\notin G$; an action $a$ is \emph{applicable} in $s$ if $\gamma(s,a)$ is defined. A plan $\pi=\langle a_1,\dots,a_T\rangle\in A^{*}$ of length $|\pi|=T$ is \emph{executable} if each $a_t$ is applicable in $s_{t-1}$, where $s_t=\gamma(s_{t-1},a_t)$, and \emph{valid} if, in addition, $s_T\in G$. $\Pi$ is \emph{solvable} if it admits a valid plan, and we then write $\Lstar(\Pi)$ for the length of a shortest valid plan.
\end{definition}

All tasks are specified in PDDL~\cite{mcdermott1998pddl}, with numeric fluents~\cite{fox2003pddl21} for the capacities in \dom{LoadBalancing}. All actions have unit cost, so plan quality is plan length. Systems never see the PDDL. Each task is presented through two natural-language briefs, $B(\Pi)$. The \emph{domain brief} describes in prose the object types, the world rules and, for every action, its parameters, requirements and effects. The \emph{instance brief} lists the objects and describes the initial situation, the goal and the optimisation metric. For example, one requirement of \texttt{assign-bucket} in the \dom{LoadBalancing} brief reads ``The server must have available space (its current load must be strictly less than its capacity).'' Briefs keep the exact object and action names, so that plans can be parsed, but never expose predicates or PDDL syntax. They were generated from the PDDL files with GPT-5 (\texttt{gpt-5-2025-08-07}, minimal reasoning effort) in two stages: first the domain brief, then each instance brief conditioned on it. The briefs were checked manually against the source files, and the generation prompts are provided in the supplementary material.

\subsection{Offline planning and planning-by-acting}
\label{sec:regimes}

An agent is commonly modelled as a function from its \emph{interaction history}, the calls it has made to its environment and the observations it has received, to its next action~\cite{russell2021aima,wooldridge2009introduction}. Internal model calls, such as the intermediate steps of a prompting strategy, do not extend this history; only calls to the environment do. This yields the classical distinction between offline problem solving, which computes a complete solution before acting, and online problem solving, which interleaves acting and observing~\cite[Sec.~4.5]{russell2021aima} (Figure~\ref{fig:protocols}).

\begin{definition}[Offline planner]
\label{def:offline}
An \emph{offline planner} maps the briefs $B(\Pi)$, possibly at random, to an output that is parsed into $y\in A^{*}\cup\{\bot\}$, where $\bot$ denotes the flag \UNSAT, the claim that $\Pi$ is unsolvable; an output that cannot be parsed in this way is \emph{malformed}. The planner may perform any number of internal model calls, but its only input from the environment is $B(\Pi)$.
\end{definition}

\begin{definition}[Online agent]
\label{def:online}
The environment of an \emph{online agent} is a simulator of $\Pi$ that maintains a current state $s$, initially $s_0$, and accepts three kinds of calls. An \emph{epistemic} call returns an observation and leaves $s$ unchanged. An \emph{ontic} call $\tool{act}(a)$, with $a\in A$, ends the episode with a \emph{precondition violation} if $a$ is not applicable in $s$; otherwise, it sets $s\leftarrow\gamma(s,a)$ and returns only an acknowledgement. The \emph{terminal} call $\tool{submit}()$ ends the episode with the claim that $s\in G$, and $\tool{submit}(\bot)$ with the claim that $\Pi$ is unsolvable. The agent is a possibly randomised policy that maps each interaction history $h_t=\langle o_0,x_1,o_1,\dots,x_t,o_t\rangle$ of calls $x_i$ and observations $o_i$ (where $o_0$ is the instruction) to the next call. An ill-formed call (an unknown tool, action or object, a wrongly typed argument, or a wrong number of arguments) returns an error observation and leaves $s$ unchanged. Every call, however many tokens it uses, consumes one unit of a budget $K(\Pi)$, which is not disclosed to the agent; exhausting it without a terminal call ends the episode with a \emph{timeout}.
\end{definition}
In our environment, the epistemic calls return a task overview (scenario, a summary of the initial situation and the goal), object listings and descriptions, the action catalogue with natural-language requirements and effects, a rendering of the current state, and the history of executed actions (Appendix~C of the supplementary material shows them for an example task, with a complete episode). Offline, a step that names a non-existent or wrongly typed object makes the plan malformed. The budget is $K(\Pi)=\max(20,3\Lstar(\Pi))$ calls for solvable tasks and 40 for unsolvable ones. This budget allows an agent to retrieve the task with four calls, query the state before every step of an optimal plan, and submit ($2\Lstar(\Pi)+5$ calls in total), with at least two calls to spare. Treating an inapplicable action as fatal encodes a safety requirement: in a deployed system, such an action would be unsafe, illegal or physically impossible. An ill-formed call, by contrast, is a failure of communication that the environment can reject harmlessly.

\subsection{Planning properties}
\label{sec:properties}

Table~\ref{tab:claims} maps each attributed competence to the property that operationalises it.

\begin{definition}[Planning properties]
\label{def:properties}
Let $\pi$ be the plan that an episode on $\Pi$ commits to: the steps of $y$ offline (the empty plan if $y=\bot$), and the well-formed $\tool{act}$ calls online. The episode is \emph{executable} if its output is not malformed and $\pi$ is executable. It \emph{succeeds} if $\Pi$ is solvable and $\pi$ is valid (online, the episode must also end with $\tool{submit}()$), or if $\Pi$ is unsolvable and the episode claims so ($y=\bot$ offline, $\tool{submit}(\bot)$ online). A correct claim of unsolvability is a \emph{detection}; an incorrect one is a \emph{false abstention}. The \emph{score} is $\Lstar(\Pi)/|\pi|$, the ratio of the optimal to the committed plan length, for a success on a solvable task, $1$ for a detection and $0$ otherwise. For online episodes, the \emph{inspection rate} is the fraction of $\tool{act}$ calls preceded by a query of the current state since the previous $\tool{act}$ call.
\end{definition}

\begin{table}[t]
\centering
\small
\caption{Competences attributed to prompting strategies and agent architectures (with the methods credited with them), the property that operationalises each, and its measure.}
\label{tab:claims}
\begin{tabular}{@{}>{\raggedright\arraybackslash}p{0.33\columnwidth}>{\raggedright\arraybackslash}p{0.34\columnwidth}>{\raggedright\arraybackslash}p{0.25\columnwidth}@{}}
\toprule
Attributed competence & Checked property & Measure \\
\midrule
Means--ends reasoning (CoT, Plan-and-Solve, ReAct) & The plan reaches the goal & Success (solvable) \\
Tracking state and constraints (Step-Back, ReAct) & Every committed action is applicable & Executability \\
Search and lookahead (Tree of Thoughts) & Higher success, no wasted actions & Success; score $\Lstar/|\pi|$ \\
Self-monitoring and repair (Self-Refine, Reflexion, Critic) & Fewer failures than the within-model baseline & Paired $\Delta$ vs.\ baseline \\
Sensing before acting (ReAct, Critic) & The state is queried before each commitment & Inspection rate; failure mode \\
Recognising infeasibility (all) & $\bot$ if and only if no plan exists & Detection; false abstention \\
Efficient reasoning (all) & Gains justify the overhead & Tokens, cost, Pareto front \\
\bottomrule
\end{tabular}
\end{table}

All measures are computed from validated plans and trajectories, never from model-generated rationales. Every failed episode on a solvable task is assigned exactly one terminal category: \emph{malformed output} (offline), \emph{precondition violation}, \emph{goal not reached} (all actions applicable, final state not in $G$), \emph{timeout} (online), or \emph{false abstention}. On an unsolvable task, every failure is a \emph{missed infeasibility}, since only $\bot$ can succeed, but we still record how the episode ended.

\subsection{Relating the two regimes}
\label{sec:dominance}

The two regimes are not symmetric. The following elementary observation makes comparisons between them interpretable.

\begin{proposition}[Acting weakly dominates offline planning]
\label{prop:dominance}
Suppose that the observations of a fixed sequence of $c$ epistemic calls, issued at the start of an episode, determine $B(\Pi)$. For every offline planner $P$, there is an online agent $\mathcal{A}_P$ such that, whenever a run of $P$ on a task $\Pi$ outputs $y$ and commits to the plan $\pi$, with $c+|\pi|+1\le K(\Pi)$, the episode of $\mathcal{A}_P$ that simulates this run achieves at least the success and score of that run.
\end{proposition}

\begin{proof}
$\mathcal{A}_P$ first issues the $c$ epistemic calls, reconstructs $B(\Pi)$ from their observations and runs $P$ on it internally, with the same random choices, to obtain $y$. If $y=\bot$, it calls $\tool{submit}(\bot)$; otherwise, it calls $\tool{act}(a_t)$ for each step $a_t$ of $y$ in turn and then $\tool{submit}()$. This uses at most $c+|\pi|+1\le K(\Pi)$ calls. If $y$ fails, its success and score are $0$. If $y=\bot$ succeeds, $\Pi$ is unsolvable and $\tool{submit}(\bot)$ is a detection. If a plan $y$ succeeds, it is valid. Since $\gamma$ is deterministic, and the simulator starts in $s_0$ and applies the same $\gamma$ against which $y$ is validated offline, every $\tool{act}$ call is applicable and the episode reaches the final state of $y$, which lies in $G$, before $\tool{submit}()$. The episode commits to the same plan $\pi$ and therefore obtains the same score.
\end{proof}

In our benchmark the premise holds with $c=4$: the overview, the object listing, the action catalogue and the current state give the objects, initial situation, goal and action rules, and so determine $\Pi$ and hence its briefs $B(\Pi)$, which are fixed for each task. Every successful offline plan we observed also satisfies the length condition. If an agent performs worse than the offline planner built on the same model, the shortfall is therefore not imposed by the environment. It lies in how the agent architecture presents the task to the model, in what the agent chooses to sense, and in when it commits to an action.

\section{Benchmark}
\label{sec:benchmark}

The suite has five domains with ten tasks each: three standard planning domains and two new ones. \dom{Blocksworld}, \dom{Gripper} and \dom{Logistics} come from the planning literature~\cite{ghallab2004automated,mcdermott2000aips98}, and each targets a different difficulty: \dom{Blocksworld} stresses subgoal interaction under strict preconditions, \dom{Gripper} bounded resources and batching with two single-capacity grippers, and \dom{Logistics} long causal chains that alternate between trucks and an aeroplane. Because these domains are widely published, we replace object identifiers with random strings (\eg \texttt{airplane\_1} becomes \texttt{air\_4f77ec}) to reduce the chance that plans are recalled rather than constructed. Action names remain meaningful, since obfuscating them~\cite{valmeekam2023planbench} would test symbol manipulation rather than the natural-language planning to which the claims refer.

The two new domains, \dom{LoadBalancing} and \dom{Hospital}, are unlikely to have been memorised, and they model applied settings in which an invalid action has real consequences. \dom{LoadBalancing} assigns traffic buckets to servers subject to region, service availability and numeric capacity. Its actions power servers and start or stop services, and its goals can require reconfiguration, such as decommissioning a server. \dom{Hospital} manages patient placement under ward suitability and isolation requirements. Beds can be dirty or blocked, and transfers and discharges leave beds dirty, so plans must account for side effects.

In the classification of Russell and Norvig~\cite{russell2021aima}, every task environment is fully observable, deterministic, static, discrete, known and single-agent: the most benign class except that tasks are sequential. Failures therefore cannot be attributed to uncertainty or to other agents. An ideal agent could track the state from the task description alone; state queries allow an agent to correct its own state tracking.

Each domain has eight solvable and two unsolvable tasks (Table~\ref{tab:instances}). Optimal plan lengths, computed by breadth-first search over the reachable state space, range from 2 to 63 actions, but length is a poor proxy for difficulty: long \dom{Gripper} plans repeat a simple pattern, while short \dom{Blocksworld} plans require reasoning about interfering subgoals. Unsolvable tasks are small modifications of solvable ones, for example removing the only truck in the destination city (\dom{Logistics}), lowering a server's capacity below the demand it must absorb (\dom{LoadBalancing}), requiring a cyclic tower (\dom{Blocksworld}), or starting with neither gripper free nor holding a ball, so that no ball can be picked up (\dom{Gripper}).

\begin{table}[t]
\centering
\small
\setlength{\tabcolsep}{3.2pt}
\renewcommand{\arraystretch}{1.2}
\caption{Optimal plan length $\Lstar$ per task; a dash marks an unsolvable task.}
\label{tab:instances}
\begin{tabular}{@{}lcccccccccc@{}}
\toprule
Domain & 1 & 2 & 3 & 4 & 5 & 6 & 7 & 8 & 9 & 10 \\
\midrule
\dom{Blocksworld}   & 4 & 6 & 6 & 10 & -- & 8 & -- & 10 & 10 & 14 \\
\dom{Gripper}       & -- & 5 & 9 & 11 & 15 & 17 & 21 & 23 & -- & 29 \\
\dom{Logistics}     & 9 & 15 & 21 & -- & -- & 39 & 45 & 51 & 57 & 63 \\
\dom{LoadBalancing} & 4 & 3 & -- & 7 & 4 & 5 & 3 & -- & 8 & 5 \\
\dom{Hospital}      & 2 & 2 & -- & 6 & 11 & -- & 9 & 8 & 5 & 10 \\
\bottomrule
\end{tabular}
\end{table}

Every offline plan is validated with VAL~\cite{howey2004val}, which also reports the first failing step and its unsatisfied preconditions. The agent environment is built on the same validator: each proposed action is checked by re-validating the executed prefix extended with the candidate, and the state returned to the agent is reconstructed from the validator's trace. Offline planners and online agents are therefore judged by identical semantics.

Before validation, model outputs must be turned into actions. An extractor parses PDDL-style action lines, with a fallback that compiles short imperative steps using the action catalogue. Agent controllers are implemented in DSPy~\cite{khattab2024dspy}, and every run logs its prompts, raw outputs, parsed plan, validator transcript, full tool-call trajectory, token counts and cost. The benchmark itself is stored as data (PDDL files plus one JSON file per domain with the briefs and tool-facing schemas), so domains and problems can be added without code changes. Code and data are provided in the supplementary material, whose Appendices~A and~B list the PDDL domains and two example tasks.
\section{Experimental Setup}
\label{sec:setup}

Our experiments address four research questions. \textbf{RQ1} (offline planning): do prompting strategies improve the executability, goal achievement and efficiency of offline plans? \textbf{RQ2} (planning-by-acting): do agents match or exceed the success rate of offline planning with the same model, which Proposition~\ref{prop:dominance} shows to be possible, and does reflective memory or a critic improve on ReAct? \textbf{RQ3} (abstention): do systems recognise unsolvable tasks without abstaining on solvable ones? \textbf{RQ4} (cost and scale): are strategies and architectures worth their overhead, and do they narrow the gap between lower- and higher-tier models?

We evaluate four models from two providers, one lower-tier and one higher-tier model from each (Table~\ref{tab:models}), which allows comparisons across tiers within each family, although the providers do not disclose parameter counts. We use non-reasoning models by design: the strategies and architectures under test were proposed for such models, and models trained for long reasoning internalise part of what these methods externalise. All calls use temperature~0 except the proposal step of Tree of Thoughts, and offline calls are capped at 4{,}096 output tokens.

\begin{table}[t]
\centering
\small
\caption{Evaluated models and list prices (USD per million tokens, as of January 2026) used for cost estimates.}
\label{tab:models}
\begin{tabular}{@{}llrr@{}}
\toprule
Model & API identifier & Input & Output \\
\midrule
GPT-4o-mini & \texttt{gpt-4o-mini-2024-07-18} & 0.15 & 0.60 \\
GPT-4o & \texttt{gpt-4o-2024-08-06} & 2.50 & 10.00 \\
Gemini 2.0 Flash-Lite & \texttt{gemini-2.0-flash-lite} & 0.075 & 0.30 \\
Gemini 2.0 Flash & \texttt{gemini-2.0-flash} & 0.10 & 0.40 \\
\bottomrule
\end{tabular}
\end{table}

\begin{table*}[t]
\centering
\small
\caption{Results per model and method. \emph{Succ.}: success on the 40 solvable tasks (\%). \emph{Exec.}: share of those episodes that are executable (well formed, with every action applicable), whether or not they reach the goal (Definition~\ref{def:properties}; \%). \emph{Det.}: unsolvable tasks recognised (out of 10). \emph{Tok.}: mean tokens per episode in thousands, averaged over the four models (per-model values are in the supplementary material). Bold marks the best success per model in each regime. No difference from the within-model baseline (Direct offline, ReAct online) survives the Holm correction; $^{\ast}$ marks $p<0.05$ before correction (exact McNemar test).}
\label{tab:results}
\setlength{\tabcolsep}{4.2pt}
\begin{tabular}{@{}lr ccc ccc ccc ccc@{}}
\toprule
& & \multicolumn{3}{c}{GPT-4o-mini} & \multicolumn{3}{c}{GPT-4o} & \multicolumn{3}{c}{Gemini 2.0 Flash-Lite} & \multicolumn{3}{c}{Gemini 2.0 Flash} \\
\cmidrule(lr){3-5}\cmidrule(lr){6-8}\cmidrule(lr){9-11}\cmidrule(l){12-14}
Method & Tok. & Succ. & Exec. & Det. & Succ. & Exec. & Det. & Succ. & Exec. & Det. & Succ. & Exec. & Det. \\
\midrule
\multicolumn{14}{@{}l}{\emph{Offline (prompting strategies)}}\\
Direct & 1.5 & 12.5 & 12.5 & 1 & 57.5 & 62.5 & 1 & \textbf{25.0} & 45.0 & 0 & 45.0 & 47.5 & 0 \\
Zero-shot CoT & 4.3 & \textbf{27.5}$^{\ast}$ & 32.5 & 1 & \textbf{77.5}$^{\ast}$ & 80.0 & 3 & 22.5 & 27.5 & 0 & 62.5 & 65.0 & 2 \\
Plan-and-Solve & 2.9 & 17.5 & 22.5 & 0 & 75.0 & 77.5 & 1 & 22.5 & 25.0 & 0 & \textbf{65.0}$^{\ast}$ & 65.0 & 0 \\
Step-Back & 3.0 & 20.0 & 25.0 & 1 & 57.5 & 57.5 & 1 & 20.0 & 30.0 & 0 & 57.5 & 60.0 & 3 \\
Self-Refine & 7.1 & 7.5 & 10.0 & 1 & 57.5 & 60.0 & 1 & 22.5 & 42.5 & 0 & 50.0 & 60.0 & 4 \\
Tree of Thoughts & 32.3 & 17.5 & 20.0 & 0 & 42.5 & 52.5 & 1 & 15.0 & 17.5 & 0 & 35.0 & 47.5 & 1 \\
\midrule
\multicolumn{14}{@{}l}{\emph{Online (agent architectures)}}\\
ReAct & 25.6 & \textbf{7.5} & 20.0 & 0 & \textbf{40.0} & 52.5 & 4 & \textbf{12.5} & 17.5 & 0 & \textbf{32.5} & 42.5 & 1 \\
Reflexion & 28.7 & \textbf{7.5} & 12.5 & 0 & 35.0 & 55.0 & 2 & \textbf{12.5} & 12.5 & 1 & 27.5 & 50.0 & 2 \\
Critic & 45.1 & 0.0 & 30.0 & 0 & 35.0 & 40.0 & 0 & 10.0 & 20.0 & 0 & \textbf{32.5} & 47.5 & 0 \\
\bottomrule
\end{tabular}
\end{table*}

Offline, each model is prompted with six strategies. All receive the same briefs, all prompts are zero-shot, and the final call of every strategy must follow the same output contract: a single fenced block with one grounded action per line, or the flag \UNSAT. If several blocks are produced, we keep the last; intermediate calls are unconstrained. \emph{Direct} makes a single call. \emph{Zero-shot CoT}~\cite{kojima2022zeroshot} first elicits a free-form reasoning trace (``Let's think step by step'') and then asks for the plan conditioned on it. \emph{Plan-and-Solve}~\cite{wang2023planandsolve} first asks for a sketch of 4--10 subgoals without grounded actions, and \emph{Step-Back}~\cite{zheng2024stepback} first abstracts the problem into 6--12 constraints and invariants.

\emph{Self-Refine}~\cite{madaan2023selfrefine} alternates drafting and critique by the same model for up to three rounds. Since LLM self-critique of plan semantics is unreliable~\cite{stechly2025selfverification}, the critique may report only mechanically checkable issues: unknown names, wrong arity or argument order, and formatting errors. The loop stops early when the critique finds no such issue or the plan stops changing. \emph{Tree of Thoughts} (ToT)~\cite{yao2023tot} runs a beam search over plan prefixes. At each of six depths, it samples three continuations of one to three actions for each prefix (temperature 0.7), scores them in a separate call on a 0--10 scale, and keeps the two best prefixes. The best prefix is finally completed into a full plan.

Online, the same models act as agents under three architectures, which give the acting model the same zero-shot instruction: it warns that an inapplicable action ends the episode, advises querying the state when unsure and explains how to declare a task unsolvable. All prompts are given in Appendix~D of the supplementary material. \emph{ReAct}~\cite{yao2023react} is a single policy that emits a reasoning step and one tool call per turn, conditioned on the full interaction history.

The other two architectures extend ReAct with the components named in RQ2: reflective memory and a critic. \emph{Reflexion}~\cite{shinn2023reflexion} is adapted to a single episode, because a precondition violation ends the episode and a deployed agent cannot retry after an unsafe act. An actor behaves as in ReAct; after every ill-formed call, a reflector writes a short lesson (\eg the correct signature of an action) into a memory that the actor sees for the rest of the episode. \emph{Critic-guided ReAct} (Critic) adapts CRITIC~\cite{gou2024critic}. Every proposed \tool{act} or \tool{submit} call is first reviewed by a critic, a separate call to the same model that sees the same interaction history. A rejection is returned to the actor, together with the critic's reason, in place of an observation. Unlike in CRITIC, the critic calls no tools, so it judges each proposal from the history alone. Every proposal, approved or rejected, counts against the budget, which bounds the exchange between actor and critic.

In total, every model is run once with every strategy and architecture on all 50 tasks, giving $4\times6\times50=1{,}200$ offline and $4\times3\times50=600$ online episodes. Closed APIs can retain some non-determinism even at temperature~0, but single runs at temperature~0 keep the comparison as reproducible as these APIs allow; our intervals and tests therefore capture variation over tasks, not over repeated runs. We report success on the 40 solvable tasks, detection on the 10 unsolvable ones, and the score over all 50, together with costs computed from logged token counts at the list prices in Table~\ref{tab:models}; for the best configurations, we give 95\% bootstrap confidence intervals over tasks. Because all conditions share the same tasks, we compare each strategy and architecture with its within-model baseline (Direct offline, ReAct online) using exact McNemar tests on paired outcomes, and we control the family-wise error rate with the Holm--Bonferroni procedure within each regime (20 comparisons offline, 8 online). With 40 paired tasks, a difference must reach at least 25 points offline (22.5 online) to survive the correction.
\section{Results}
\label{sec:results}

\begin{figure*}[t]
\centering
\begin{minipage}[b]{0.485\textwidth}
\centering
\includegraphics[width=\linewidth]{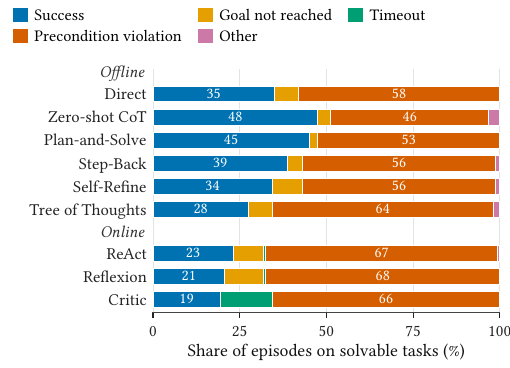}\\[2pt]
{\small (a) How episodes on solvable tasks end}
\end{minipage}\hfill
\begin{minipage}[b]{0.485\textwidth}
\centering
\includegraphics[width=\linewidth]{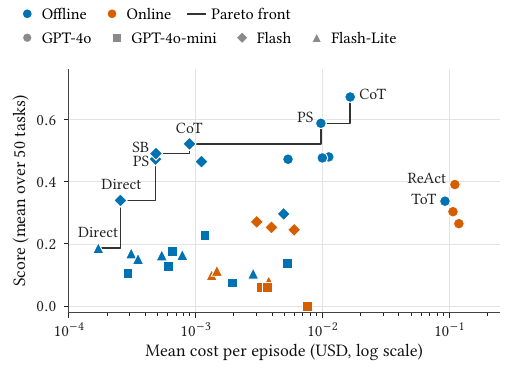}\\[2pt]
{\small (b) Score against cost}
\end{minipage}
\caption{Outcomes and cost. (a) Episodes on solvable tasks, pooled over the four models (160 per method); labels give the percentage of successes and of precondition violations, and \emph{Other} covers malformed outputs and false abstentions. (b) Score against mean cost per episode for the 36 configurations; colour marks the regime, marker shape the model, and the line the Pareto front. PS: Plan-and-Solve; SB: Step-Back; CoT: chain-of-thought; ToT: Tree of Thoughts.}
\label{fig:results}
\Description{Two panels. Left: horizontal stacked bars, one per method. Offline, Direct, zero-shot CoT, Plan-and-Solve, Step-Back, Self-Refine and Tree of Thoughts succeed in 35, 48, 45, 39, 34 and 28 percent of episodes, and precondition violations end 58, 46, 53, 56, 56 and 64 percent. Online, ReAct, Reflexion and Critic succeed in 23, 21 and 19 percent, and precondition violations end 67, 68 and 66 percent; Critic also has a 15 percent share of timeouts. Right: scatter plot of score against cost on a logarithmic axis. Offline configurations, shown in blue, form the upper-left boundary; online configurations, shown in red, all lie below it. GPT-4o with chain-of-thought has the highest score, 0.67, at about 0.017 dollars per episode. GPT-4o with Tree of Thoughts and GPT-4o with ReAct cost about 0.1 dollars per episode and score 0.34 and 0.39.}
\end{figure*}

Table~\ref{tab:results} gives the results of every configuration, and Figure~\ref{fig:results} shows how episodes end and what they cost.

\subsection{Offline planning (RQ1)}
\label{sec:res-offline}

The best offline configuration, GPT-4o with zero-shot chain-of-thought, solves 77.5\% of the solvable tasks (31 of 40; 95\% bootstrap interval 65--90\%). Gains from prompting are model-specific. Chain-of-thought and Plan-and-Solve raise the success rate of GPT-4o from 57.5\% to 77.5\% and 75.0\%, and that of Gemini~2.0 Flash from 45.0\% to 62.5\% and 65.0\%. Chain-of-thought also raises GPT-4o-mini's success rate from 12.5\% to 27.5\%. In contrast, no strategy improves on direct prompting for Flash-Lite. Self-Refine never exceeds direct prompting by more than five points, and Tree of Thoughts lowers success for three of the four models.

None of these differences survives the Holm correction. Three are significant before correction: chain-of-thought for GPT-4o (9 tasks gained, 1 lost; $p=0.02$) and for GPT-4o-mini (6 gained, none lost; $p=0.03$), and Plan-and-Solve for Flash (10 gained, 2 lost; $p=0.04$). With 40 tasks, only large effects are detectable (Section~\ref{sec:setup}), so these gains may still be genuine.

Precondition violations are the dominant failure mode. Under every strategy, 46--64\% of the plans for solvable tasks contain an inapplicable action (Figure~\ref{fig:results}a), and such plans account for 89\% of the failed ones. Executable plans that miss the goal are less frequent (2.5--9\% per strategy), and malformed outputs are rarer still (1.3\% overall). Consequently, executability is close to success for the strongest configurations (80.0\% against 77.5\% for GPT-4o with chain-of-thought). The gap is larger for some weaker configurations: the direct plans of Flash-Lite are executable in 45.0\% of the solvable episodes but reach the goal in only 25.0\%. Successful plans are usually optimal: 83\% of them have optimal length, and the mean score of a success is 0.95.

\subsection{Online planning-by-acting (RQ2)}
\label{sec:res-online}

Every model performs worse online than offline. The best online configuration, GPT-4o with ReAct, solves 40.0\% of the solvable tasks (16 of 40; 95\% interval 25--55\%), against 77.5\% offline. For GPT-4o, Flash and Flash-Lite, the best online result is about half of the best offline one; for GPT-4o-mini, it is about a quarter (7.5\% against 27.5\%). The drop already appears between the baselines: ReAct falls below direct prompting for all four models. Pooled over models and methods, the gap appears in every domain (9 to 23 points), and the new domains are not systematically harder than the classical ones (Appendix~F of the supplementary material). By Proposition~\ref{prop:dominance}, the gap is not caused by the environment, because every successful offline plan fits within the budget.

Agents rarely use the online regime's queries before acting. Only 7.3\% of their actions are preceded by a query of the current state (the inspection rate; ReAct 6.4\%, Reflexion 8.7\%, Critic 6.7\%), and in 58\% of the episodes in which an agent acts at all, it never queries the state. An average episode contains 5.8 actions but only 0.5 state queries, even though the instruction advises the agent to query the state when unsure. Precondition violations terminate 66--68\% of the solvable episodes under every architecture and account for 84\% of the failed ones (77--94\% per model).

Neither reflective memory nor a critic improves performance. Reflexion and Critic never outperform ReAct for any model, and no difference survives the correction (smallest uncorrected $p=0.25$). The architectures alter how agents fail rather than how often. The critic leaves precondition violations almost unchanged (66\% of solvable episodes, against 67\% for ReAct) but adds timeouts: 15.0\% of its solvable episodes exhaust the budget, against 0.6\% for ReAct, because every rejected proposal consumes a call. Reflexion improves executability over ReAct for the two higher-tier models, but more of its episodes end without reaching the goal.

\subsection{Abstention on unsolvable tasks (RQ3)}
\label{sec:res-infeasible}

Systems rarely recognise that a task is unsolvable. Only 32 of the 360 episodes on unsolvable tasks end with a detection (8.9\%): 22 of 240 offline and 10 of 120 online. Online, about nine in ten missed infeasibilities end with a precondition violation. The best configurations recognise 4 of the 10 tasks (Flash with Self-Refine offline, GPT-4o with ReAct online). Critic recognises none. Detections are also concentrated on a few tasks: 13 of the 22 offline detections occur on a single \dom{Gripper} task in which no gripper can ever pick up a ball. False abstentions, by contrast, are almost absent: only 1 of the 1,440 episodes on solvable tasks claims that no plan exists. The deficit therefore lies in recall rather than precision. Systems tend to attempt a plan rather than abstain, and no strategy or architecture reverses this tendency.

\subsection{Cost and scale (RQ4)}
\label{sec:res-cost}

The most expensive strategies and architectures do not produce gains that justify their cost. Relative to direct prompting, Plan-and-Solve and Step-Back use about twice as many tokens, chain-of-thought about three times, Self-Refine about five times and Tree of Thoughts about 22 times as many. Online, Reflexion uses 1.1 times and Critic 1.8 times as many tokens as ReAct, and an online episode costs \$0.031 on average, against \$0.007 offline. Every configuration on the Pareto front of score against cost is offline (Figure~\ref{fig:results}b), and Tree of Thoughts, Self-Refine and all online configurations are dominated.

Chain-of-thought and Plan-and-Solve widen the gap between model tiers rather than narrowing it. GPT-4o leads GPT-4o-mini by 45 points with direct prompting and by 50 points with chain-of-thought, and Flash leads Flash-Lite by 20 points with direct prompting and by 40 points with chain-of-thought. Averaged over strategies, offline success ranges from 17\% to 61\% across models; averaged over models, it ranges only from 28\% to 48\% across strategies. Online, the corresponding ranges are 5--37\% across models and 19--23\% across architectures. Success therefore varies more across models than across prompting strategies and architectures.

\section{Discussion and Limitations}
\label{sec:discussion}

The results support few of the competences attributed to the strategies and architectures in Table~\ref{tab:claims}. Chain-of-thought and Plan-and-Solve improve means--ends reasoning and constraint tracking for three of the four models, but by margins that are too small to confirm on 40 tasks. Search-based prompting does not improve planning: Tree of Thoughts reproduces the form of search without the systematic exploration of the state space on which the reliability of a planner rests. Self-monitoring architectures do not reduce failures, none induces sensing before acting, and simple strategies dominate expensive ones in cost-effectiveness.

Acting should have worked in the agents' favour: by Proposition~\ref{prop:dominance}, an agent in our benchmark can always plan first and then execute. The agents we evaluated nevertheless perform worse online, and most of their failures result from inapplicable actions that a state query would have exposed. Rare sensing explains why agents do not benefit from online queries, but not why they perform worse than offline planners, which do not sense at all. A more likely cause is overconfident commitment: agents commit to one action at a time before they have a complete plan, based on beliefs about the state that they rarely check. In this respect, they resemble the bold agents of Kinny and Georgeff~\cite{kinny1991commitment}, which execute their plans without stopping to reconsider.

Our critic does not remedy this overconfidence, since it reads the same history as the actor and cannot query the state itself. These results suggest that the model should not decide whether to monitor execution: a sound validator should intercept inapplicable actions before execution, as LLM-Modulo architectures propose~\cite{kambhampati2024llmmodulo}. In our environment, such a gate would turn every precondition violation into a recoverable error; whether agents would then recover remains an open question.

A rational agent drops an intention once it believes the goal to be unachievable~\cite{cohen1990intention}, but the systems we evaluated almost never do so. GPT-4o and Flash together account for 27 of the 32 detections, yet even these models miss most unsolvable tasks. Planning ability therefore does not guarantee the recognition of infeasibility, which must be tested directly.

Our results suggest four practices for substantiating claims about agentic planning: report executability, failure modes and detection rather than success alone; state whether a system is offline or online in the agent-theoretic sense, since the two regimes test different competences; include tasks that admit no valid plan, without which an inability to abstain remains invisible; and report cost, because complex designs are often dominated by simple ones~\cite{kapoor2025agents}. Where an explicit model of the task exists, planning formalisms make such checks inexpensive.

The benchmark is small (50 tasks), and each configuration was run once, at temperature~0.
Its domains are deterministic, fully observable and single-agent, which excludes partial observability, stochastic effects and interaction with other agents. Three of the five domains are classical, and randomised identifiers reduce, but do not eliminate, the advantage of familiarity. The briefs were generated with an LLM and checked by hand, so a residual translation error cannot be ruled out. Treating precondition violations as fatal is a deliberate safety choice; a regime that merely reports them would favour online agents.  Multi-agent LLM systems add coordination and communication, which our semantics does not cover; multi-agent planning domains are a natural next step.

The results are specific to the evaluated closed models, whose serving behaviour may change over time, and models trained for long reasoning may behave differently. Finally, we re-implemented all strategies and architectures and adapted some of them to our protocol (Section~\ref{sec:setup}): Reflexion reflects within one episode rather than across retries; Self-Refine critiques only mechanically checkable issues, so semantic repair is tested only by Critic; and in Critic, the critic calls no tools. Other prompts or hyperparameters could change individual results.

\section{Conclusion}
\label{sec:conclusion}


We asked whether prompting strategies and agent architectures give LLM-based systems the planning competences claimed for them. For the four models we evaluated, and planning problems with a short length (2 to 30 actions) the evidence is largely negative. LLM-based planning ignores action preconditions and fails to reach the goal in the majority of cases, with performance differences across architectures and prompting strategies being model-specific and statistically fragile: no one strategy solves this problem. In the cases where the agent can choose to inspect the state, the agent foregoes this option to commit to potentially dangerous or impossible actions. Moreover, despite being informed of the possibility of problems being impossible, these tasks are rarely recognised and the LLM wastes resources in trying to solve them. The most expensive strategies increase cost by up to 20 to 30 times without improving average success. Planning performance appears to depend mainly on the underlying model. A strong advice to agent designers is to use validators to block inapplicable actions, as it is the most frequent mode of failure; and that evaluations should check planning properties rather than end-to-end success alone.

\bibliographystyle{ACM-Reference-Format}
\bibliography{refs}

@article{wooldridge1995intelligent,
  title     = {Intelligent Agents: Theory and Practice},
  author    = {Wooldridge, Michael and Jennings, Nicholas R.},
  journal   = {The Knowledge Engineering Review},
  volume    = {10},
  number    = {2},
  pages     = {115--152},
  year      = {1995}
}

@inproceedings{franklin1997agent,
  title     = {Is It an Agent, or Just a Program?: A Taxonomy for Autonomous Agents},
  author    = {Franklin, Stan and Graesser, Art},
  booktitle = {Intelligent Agents III: Agent Theories, Architectures, and Languages (ATAL 1996)},
  series    = {Lecture Notes in Computer Science},
  volume    = {1193},
  publisher = {Springer},
  year      = {1997},
  doi       = {10.1007/BFb0013570}
}

@book{russell2021aima,
  title     = {Artificial Intelligence: A Modern Approach},
  author    = {Russell, Stuart J. and Norvig, Peter},
  edition   = {4th},
  publisher = {Pearson},
  year      = {2021}
}

@inproceedings{rao1995bdi,
  title     = {{BDI} Agents: From Theory to Practice},
  author    = {Rao, Anand S. and Georgeff, Michael P.},
  booktitle = {Proceedings of the First International Conference on Multi-Agent Systems (ICMAS)},
  pages     = {312--319},
  year      = {1995},
  publisher = {AAAI Press},
}

@inproceedings{wooldridge2026rethinking,
  title     = {Rethinking Multi-agent Systems in the Era of {LLMs}},
  author    = {Wooldridge, Michael},
  booktitle = {Proceedings of the 25th International Conference on Autonomous Agents and Multiagent Systems (AAMAS)},
  year      = {2026},
  publisher = {International Foundation for Autonomous Agents and Multiagent Systems},
  address   = {Richland, SC},
}

@misc{dignum2025agentifying,
  title         = {Agentifying Agentic {AI}},
  author        = {Dignum, Virginia and Dignum, Frank},
  year          = {2025},
  eprint        = {2511.17332},
  archivePrefix = {arXiv}
}

@inproceedings{ichida2024natbdi,
  title     = {{BDI} Agents in Natural Language Environments},
  author    = {Ichida, Alexandre Yukio and Meneguzzi, Felipe and Cardoso, Rafael C.},
  booktitle = {Proceedings of the 23rd International Conference on Autonomous Agents and Multiagent Systems (AAMAS)},
  pages     = {880--888},
  year      = {2024},
  publisher = {International Foundation for Autonomous Agents and Multiagent Systems},
  address   = {Richland, SC},
}

@inproceedings{choi2026reactree,
  title     = {{ReAcTree}: Hierarchical {LLM} Agent Trees with Control Flow for Long-Horizon Task Planning},
  author    = {Choi, Jae-Woo and Kim, Hyungmin and Ong, Hyobin and Yoon, Youngwoo and Jang, Minsu and Kim, Dohyung and Kim, Jaehong},
  booktitle = {Proceedings of the 25th International Conference on Autonomous Agents and Multiagent Systems (AAMAS)},
  pages     = {319--328},
  year      = {2026},
  publisher = {International Foundation for Autonomous Agents and Multiagent Systems},
  address   = {Richland, SC},
}

@inproceedings{hsiao2026procedural,
  title     = {Procedural Knowledge Improves Agentic {LLM} Workflows},
  author    = {Hsiao, Vincent and Roberts, Mark and Smith, Leslie},
  booktitle = {Proceedings of the 25th International Conference on Autonomous Agents and Multiagent Systems (AAMAS)},
  pages     = {1425--1433},
  year      = {2026},
  publisher = {International Foundation for Autonomous Agents and Multiagent Systems},
  address   = {Richland, SC},
}

@article{bolander2011epistemic,
  title     = {Epistemic Planning for Single- and Multi-Agent Systems},
  author    = {Bolander, Thomas and Andersen, Mikkel Birkegaard},
  journal   = {Journal of Applied Non-Classical Logics},
  volume    = {21},
  number    = {1},
  pages     = {9--34},
  year      = {2011},
  doi       = {10.3166/jancl.21.9-34}
}

@book{ghallab2016acting,
  title     = {Automated Planning and Acting},
  author    = {Ghallab, Malik and Nau, Dana and Traverso, Paolo},
  publisher = {Cambridge University Press},
  year      = {2016}
}

@book{ghallab2004automated,
  title     = {Automated Planning: Theory and Practice},
  author    = {Ghallab, Malik and Nau, Dana and Traverso, Paolo},
  publisher = {Morgan Kaufmann},
  year      = {2004}
}

@book{geffner2013concise,
  title     = {A Concise Introduction to Models and Methods for Automated Planning},
  author    = {Geffner, Hector and Bonet, Blai},
  publisher = {Morgan \& Claypool},
  series    = {Synthesis Lectures on Artificial Intelligence and Machine Learning},
  year      = {2013}
}

@techreport{mcdermott1998pddl,
  title       = {{PDDL}---The Planning Domain Definition Language, Version 1.2},
  author      = {McDermott, Drew and Ghallab, Malik and Howe, Adele and Knoblock, Craig and Ram, Ashwin and Veloso, Manuela and Weld, Daniel and Wilkins, David},
  institution = {Yale Center for Computational Vision and Control},
  number      = {CVC TR-98-003},
  year        = {1998}
}

@article{fox2003pddl21,
  title   = {{PDDL2.1}: An Extension to {PDDL} for Expressing Temporal Planning Domains},
  author  = {Fox, Maria and Long, Derek},
  journal = {Journal of Artificial Intelligence Research},
  volume  = {20},
  pages   = {61--124},
  year    = {2003}
}

@inproceedings{howey2004val,
  title     = {{VAL}: Automatic Plan Validation, Continuous Effects and Mixed Initiative Planning Using {PDDL}},
  author    = {Howey, Richard and Long, Derek and Fox, Maria},
  booktitle = {Proceedings of the 16th IEEE International Conference on Tools with Artificial Intelligence (ICTAI)},
  pages     = {294--301},
  year      = {2004},
  doi       = {10.1109/ICTAI.2004.120},
  publisher = {IEEE Computer Society},
}

@inproceedings{eriksson2017unsolvability,
  title     = {Unsolvability Certificates for Classical Planning},
  author    = {Eriksson, Salom{\'e} and R{\"o}ger, Gabriele and Helmert, Malte},
  booktitle = {Proceedings of the 27th International Conference on Automated Planning and Scheduling (ICAPS)},
  year      = {2017},
  doi       = {10.1609/icaps.v27i1.13818}
}

@book{bratman1987intention,
  title     = {Intention, Plans, and Practical Reason},
  author    = {Bratman, Michael E.},
  publisher = {Harvard University Press},
  address   = {Cambridge, MA},
  year      = {1987}
}

@book{wooldridge2009introduction,
  title     = {An Introduction to MultiAgent Systems},
  author    = {Wooldridge, Michael},
  edition   = {2nd},
  publisher = {John Wiley \& Sons},
  year      = {2009}
}

@article{meneguzzi2015planning,
  title     = {Planning in {BDI} Agents: A Survey of the Integration of Planning Algorithms and Agent Reasoning},
  author    = {Meneguzzi, Felipe and de Silva, Lavindra},
  journal   = {The Knowledge Engineering Review},
  volume    = {30},
  number    = {1},
  pages     = {1--44},
  year      = {2015},
  doi       = {10.1017/S0269888913000337}
}

@inproceedings{valmeekam2023planning,
  title     = {On the Planning Abilities of Large Language Models---A Critical Investigation},
  author    = {Valmeekam, Karthik and Marquez, Matthew and Sreedharan, Sarath and Kambhampati, Subbarao},
  booktitle = {Advances in Neural Information Processing Systems},
  volume    = {36},
  year      = {2023},
  publisher = {Curran Associates, Inc.},
}

@inproceedings{valmeekam2023planbench,
  title     = {{PlanBench}: An Extensible Benchmark for Evaluating Large Language Models on Planning and Reasoning about Change},
  author    = {Valmeekam, Karthik and Marquez, Matthew and Olmo, Alberto and Sreedharan, Sarath and Kambhampati, Subbarao},
  booktitle = {Advances in Neural Information Processing Systems (Datasets and Benchmarks Track)},
  volume    = {36},
  year      = {2023},
  publisher = {Curran Associates, Inc.},
}

@misc{valmeekam2024lrm,
  title         = {{LLMs} Still Can't Plan; Can {LRMs}? {A} Preliminary Evaluation of {OpenAI's} o1 on {PlanBench}},
  author        = {Valmeekam, Karthik and Stechly, Kaya and Kambhampati, Subbarao},
  year          = {2024},
  eprint        = {2409.13373},
  archivePrefix = {arXiv}
}

@inproceedings{kambhampati2024llmmodulo,
  title     = {Position: {LLMs} Can't Plan, But Can Help Planning in {LLM-Modulo} Frameworks},
  author    = {Kambhampati, Subbarao and Valmeekam, Karthik and Guan, Lin and Verma, Mudit and Stechly, Kaya and Bhambri, Siddhant and Saldyt, Lucas and Murthy, Anil},
  booktitle = {Proceedings of the 41st International Conference on Machine Learning (ICML)},
  series    = {Proceedings of Machine Learning Research},
  volume    = {235},
  year      = {2024},
  publisher = {PMLR},
}

@inproceedings{stechly2024thoughtlessness,
  title     = {Chain of Thoughtlessness? {An} Analysis of {CoT} in Planning},
  author    = {Stechly, Kaya and Valmeekam, Karthik and Kambhampati, Subbarao},
  booktitle = {Advances in Neural Information Processing Systems},
  volume    = {37},
  year      = {2024},
  publisher = {Curran Associates, Inc.},
}

@inproceedings{stechly2025selfverification,
  title     = {On the Self-Verification Limitations of Large Language Models on Reasoning and Planning Tasks},
  author    = {Stechly, Kaya and Valmeekam, Karthik and Kambhampati, Subbarao},
  booktitle = {International Conference on Learning Representations (ICLR)},
  year      = {2025},
  publisher = {OpenReview.net},
}

@article{verma2024brittle,
  title   = {Do Think Tags Really Help {LLMs} Plan? {A} Critical Evaluation of {ReAct}-Style Prompting},
  author  = {Bhambri, Siddhant and Verma, Mudit and Kambhampati, Subbarao},
  journal = {Transactions on Machine Learning Research},
  year    = {2025},
  url     = {https://openreview.net/forum?id=aFAMPSmNHR},
}

@inproceedings{stein2025autoplanbench,
  title     = {Automating the Generation of Prompts for {LLM}-based Action Choice in {PDDL} Planning},
  author    = {Stein, Katharina and Fi{\v{s}}er, Daniel and Hoffmann, J{\"o}rg and Koller, Alexander},
  booktitle = {Proceedings of the 35th International Conference on Automated Planning and Scheduling (ICAPS)},
  year      = {2025},
  publisher = {AAAI Press},
}

@inproceedings{kokel2025acpbench,
  title     = {{ACPBench}: Reasoning About Action, Change, and Planning},
  author    = {Kokel, Harsha and Katz, Michael and Srinivas, Kavitha and Sohrabi, Shirin},
  booktitle = {Proceedings of the 39th AAAI Conference on Artificial Intelligence},
  volume    = {39},
  number    = {25},
  pages     = {26559--26568},
  year      = {2025},
  doi       = {10.1609/aaai.v39i25.34857},
  publisher = {AAAI Press},
}

@misc{gobel2026stepwise,
  title         = {Agentic {LLM} Planning via Step-Wise {PDDL} Simulation: An Empirical Characterisation},
  author        = {G{\"o}bel, Kai and Lorang, Pierrick and Zips, Patrik and Gl{\"u}ck, Tobias},
  year          = {2026},
  eprint        = {2603.06064},
  archivePrefix = {arXiv}
}

@misc{sun2026apb,
  title         = {Agent Planning Benchmark: A Diagnostic Framework for Planning Capabilities in {LLM} Agents},
  author        = {Sun, Haoyu and Wang, Wenxuan and Song, Mingyang and He, Jujie and Zhang, Weinan and Liu, Yang and Yang, Yang and Cheng, Yu},
  year          = {2026},
  eprint        = {2606.04874},
  archivePrefix = {arXiv}
}

@inproceedings{silver2024generalized,
  title     = {Generalized Planning in {PDDL} Domains with Pretrained Large Language Models},
  author    = {Silver, Tom and Dan, Soham and Srinivas, Kavitha and Tenenbaum, Joshua B. and Kaelbling, Leslie and Katz, Michael},
  booktitle = {Proceedings of the 38th AAAI Conference on Artificial Intelligence},
  volume    = {38},
  number    = {18},
  year      = {2024},
  doi       = {10.1609/aaai.v38i18.30006}
}

@inproceedings{wei2022cot,
  title     = {Chain-of-Thought Prompting Elicits Reasoning in Large Language Models},
  author    = {Wei, Jason and Wang, Xuezhi and Schuurmans, Dale and Bosma, Maarten and Ichter, Brian and Xia, Fei and Chi, Ed and Le, Quoc and Zhou, Denny},
  booktitle = {Advances in Neural Information Processing Systems},
  volume    = {35},
  year      = {2022},
  publisher = {Curran Associates, Inc.},
}

@inproceedings{kojima2022zeroshot,
  title     = {Large Language Models are Zero-Shot Reasoners},
  author    = {Kojima, Takeshi and Gu, Shixiang Shane and Reid, Machel and Matsuo, Yutaka and Iwasawa, Yusuke},
  booktitle = {Advances in Neural Information Processing Systems},
  volume    = {35},
  year      = {2022},
  publisher = {Curran Associates, Inc.},
}

@inproceedings{wang2023planandsolve,
  title     = {Plan-and-Solve Prompting: Improving Zero-Shot Chain-of-Thought Reasoning by Large Language Models},
  author    = {Wang, Lei and Xu, Wanyu and Lan, Yihuai and Hu, Zhiqiang and Lan, Yunshi and Lee, Roy Ka-Wei and Lim, Ee-Peng},
  booktitle = {Proceedings of the 61st Annual Meeting of the Association for Computational Linguistics (ACL)},
  year      = {2023}
}

@inproceedings{zheng2024stepback,
  title     = {Take a Step Back: Evoking Reasoning via Abstraction in Large Language Models},
  author    = {Zheng, Huaixiu Steven and Mishra, Swaroop and Chen, Xinyun and Cheng, Heng-Tze and Chi, Ed H. and Le, Quoc V. and Zhou, Denny},
  booktitle = {International Conference on Learning Representations (ICLR)},
  year      = {2024},
  publisher = {OpenReview.net},
}

@inproceedings{madaan2023selfrefine,
  title     = {Self-Refine: Iterative Refinement with Self-Feedback},
  author    = {Madaan, Aman and Tandon, Niket and Gupta, Prakhar and Hallinan, Skyler and Gao, Luyu and Wiegreffe, Sarah and Alon, Uri and Dziri, Nouha and Prabhumoye, Shrimai and Yang, Yiming and Gupta, Shashank and Majumder, Bodhisattwa Prasad and Hermann, Katherine and Welleck, Sean and Yazdanbakhsh, Amir and Clark, Peter},
  booktitle = {Advances in Neural Information Processing Systems},
  volume    = {36},
  year      = {2023},
  publisher = {Curran Associates, Inc.},
}

@inproceedings{yao2023tot,
  title     = {Tree of Thoughts: Deliberate Problem Solving with Large Language Models},
  author    = {Yao, Shunyu and Yu, Dian and Zhao, Jeffrey and Shafran, Izhak and Griffiths, Thomas L. and Cao, Yuan and Narasimhan, Karthik},
  booktitle = {Advances in Neural Information Processing Systems},
  volume    = {36},
  year      = {2023},
  publisher = {Curran Associates, Inc.},
}

@inproceedings{yao2023react,
  title     = {{ReAct}: Synergizing Reasoning and Acting in Language Models},
  author    = {Yao, Shunyu and Zhao, Jeffrey and Yu, Dian and Du, Nan and Shafran, Izhak and Narasimhan, Karthik and Cao, Yuan},
  booktitle = {International Conference on Learning Representations (ICLR)},
  year      = {2023},
  publisher = {OpenReview.net},
}

@inproceedings{shinn2023reflexion,
  title     = {Reflexion: Language Agents with Verbal Reinforcement Learning},
  author    = {Shinn, Noah and Cassano, Federico and Gopinath, Ashwin and Narasimhan, Karthik and Yao, Shunyu},
  booktitle = {Advances in Neural Information Processing Systems},
  volume    = {36},
  year      = {2023},
  publisher = {Curran Associates, Inc.},
}

@inproceedings{gou2024critic,
  title     = {{CRITIC}: Large Language Models Can Self-Correct with Tool-Interactive Critiquing},
  author    = {Gou, Zhibin and Shao, Zhihong and Gong, Yeyun and Shen, Yelong and Yang, Yujiu and Duan, Nan and Chen, Weizhu},
  booktitle = {International Conference on Learning Representations (ICLR)},
  year      = {2024},
  publisher = {OpenReview.net},
}

@inproceedings{huang2024selfcorrect,
  title     = {Large Language Models Cannot Self-Correct Reasoning Yet},
  author    = {Huang, Jie and Chen, Xinyun and Mishra, Swaroop and Zheng, Huaixiu Steven and Yu, Adams Wei and Song, Xinying and Zhou, Denny},
  booktitle = {International Conference on Learning Representations (ICLR)},
  year      = {2024},
  publisher = {OpenReview.net},
}

@inproceedings{khattab2024dspy,
  title     = {{DSPy}: Compiling Declarative Language Model Calls into State-of-the-Art Pipelines},
  author    = {Khattab, Omar and Singhvi, Arnav and Maheshwari, Paridhi and Zhang, Zhiyuan and Santhanam, Keshav and Vardhamanan, Sri and Haq, Saiful and Sharma, Ashutosh and Joshi, Thomas T. and Moazam, Hanna and Miller, Heather and Zaharia, Matei and Potts, Christopher},
  booktitle = {International Conference on Learning Representations (ICLR)},
  year      = {2024}
}

@inproceedings{liu2024agentbench,
  title     = {{AgentBench}: Evaluating {LLMs} as Agents},
  author    = {Liu, Xiao and Yu, Hao and Zhang, Hanchen and Xu, Yifan and Lei, Xuanyu and Lai, Hanyu and Gu, Yu and Ding, Hangliang and Men, Kaiwen and Yang, Kejuan and Zhang, Shudan and Deng, Xiang and Zeng, Aohan and Du, Zhengxiao and Zhang, Chenhui and Shen, Sheng and Zhang, Tianjun and Su, Yu and Sun, Huan and Huang, Minlie and Dong, Yuxiao and Tang, Jie},
  booktitle = {International Conference on Learning Representations (ICLR)},
  year      = {2024},
  publisher = {OpenReview.net},
}

@inproceedings{zhou2024webarena,
  title     = {{WebArena}: A Realistic Web Environment for Building Autonomous Agents},
  author    = {Zhou, Shuyan and Xu, Frank F. and Zhu, Hao and Zhou, Xuhui and Lo, Robert and Sridhar, Abishek and Cheng, Xianyi and Ou, Tianyue and Bisk, Yonatan and Fried, Daniel and Alon, Uri and Neubig, Graham},
  booktitle = {International Conference on Learning Representations (ICLR)},
  year      = {2024},
  publisher = {OpenReview.net},
}

@inproceedings{xie2024osworld,
  title     = {{OSWorld}: Benchmarking Multimodal Agents for Open-Ended Tasks in Real Computer Environments},
  author    = {Xie, Tianbao and Zhang, Danyang and Chen, Jixuan and Li, Xiaochuan and Zhao, Siheng and Cao, Ruisheng and Hua, Toh Jing and Cheng, Zhoujun and Shin, Dongchan and Lei, Fangyu and Liu, Yitao and Xu, Yiheng and Zhou, Shuyan and Savarese, Silvio and Xiong, Caiming and Zhong, Victor and Yu, Tao},
  booktitle = {Advances in Neural Information Processing Systems (Datasets and Benchmarks Track)},
  volume    = {37},
  year      = {2024},
  publisher = {Curran Associates, Inc.},
}

@inproceedings{jimenez2024swebench,
  title     = {{SWE}-bench: Can Language Models Resolve Real-World {GitHub} Issues?},
  author    = {Jimenez, Carlos E. and Yang, John and Wettig, Alexander and Yao, Shunyu and Pei, Kexin and Press, Ofir and Narasimhan, Karthik},
  booktitle = {International Conference on Learning Representations (ICLR)},
  year      = {2024},
  publisher = {OpenReview.net},
}

@inproceedings{mialon2024gaia,
  title     = {{GAIA}: A Benchmark for General {AI} Assistants},
  author    = {Mialon, Gr{\'e}goire and Fourrier, Cl{\'e}mentine and Swift, Craig and Wolf, Thomas and LeCun, Yann and Scialom, Thomas},
  booktitle = {International Conference on Learning Representations (ICLR)},
  year      = {2024},
  publisher = {OpenReview.net},
}

@inproceedings{shridhar2021alfworld,
  title     = {{ALFWorld}: Aligning Text and Embodied Environments for Interactive Learning},
  author    = {Shridhar, Mohit and Yuan, Xingdi and C{\^o}t{\'e}, Marc-Alexandre and Bisk, Yonatan and Trischler, Adam and Hausknecht, Matthew},
  booktitle = {International Conference on Learning Representations (ICLR)},
  year      = {2021},
  publisher = {OpenReview.net},
}

@inproceedings{ma2024agentboard,
  title     = {{AgentBoard}: An Analytical Evaluation Board of Multi-turn {LLM} Agents},
  author    = {Ma, Chang and Zhang, Junlei and Zhu, Zhihao and Yang, Cheng and Yang, Yujiu and Jin, Yaohui and Lan, Zhenzhong and Kong, Lingpeng and He, Junxian},
  booktitle = {Advances in Neural Information Processing Systems (Datasets and Benchmarks Track)},
  volume    = {37},
  year      = {2024},
  publisher = {Curran Associates, Inc.},
}

@inproceedings{yehudai2025surveyagents,
  title     = {A Survey on Evaluation of {LLM}-based Agents},
  author    = {Yehudai, Asaf and Eden, Lilach and Li, Alan and Uziel, Guy and Zhao, Yilun and Bar-Haim, Roy and Cohan, Arman and Shmueli-Scheuer, Michal},
  booktitle = {Findings of the Association for Computational Linguistics: ACL 2026},
  pages     = {26690--26714},
  year      = {2026},
  publisher = {Association for Computational Linguistics},
  address   = {San Diego, California, United States},
  doi       = {10.18653/v1/2026.findings-acl.1330},
}

@article{kapoor2025agents,
  title   = {{AI} Agents That Matter},
  author  = {Kapoor, Sayash and Stroebl, Benedikt and Siegel, Zachary S. and Nadgir, Nitya and Narayanan, Arvind},
  journal = {Transactions on Machine Learning Research},
  year    = {2025}
}

@inproceedings{zheng2023judging,
  title     = {Judging {LLM}-as-a-Judge with {MT-Bench} and {Chatbot Arena}},
  author    = {Zheng, Lianmin and Chiang, Wei-Lin and Sheng, Ying and Zhuang, Siyuan and Wu, Zhanghao and Zhuang, Yonghao and Lin, Zi and Li, Zhuohan and Li, Dacheng and Xing, Eric P. and Zhang, Hao and Gonzalez, Joseph E. and Stoica, Ion},
  booktitle = {Advances in Neural Information Processing Systems (Datasets and Benchmarks Track)},
  volume    = {36},
  year      = {2023},
  publisher = {Curran Associates, Inc.},
}

@inproceedings{thakur2025judging,
  title     = {Judging the Judges: Evaluating Alignment and Vulnerabilities in {LLMs}-as-Judges},
  author    = {Thakur, Aman Singh and Choudhary, Kartik and Ramayapally, Venkat Srinik and Vaidyanathan, Sankaran and Hupkes, Dieuwke},
  booktitle = {Proceedings of the Fourth Workshop on Generation, Evaluation and Metrics (GEM)},
  pages     = {404--430},
  publisher = {Association for Computational Linguistics},
  year      = {2025}
}

@inproceedings{turpin2023unfaithful,
  title     = {Language Models Don't Always Say What They Think: Unfaithful Explanations in Chain-of-Thought Prompting},
  author    = {Turpin, Miles and Michael, Julian and Perez, Ethan and Bowman, Samuel R.},
  booktitle = {Advances in Neural Information Processing Systems},
  volume    = {36},
  year      = {2023},
  publisher = {Curran Associates, Inc.},
}

@inproceedings{kirichenko2025abstentionbench,
  title     = {{AbstentionBench}: Reasoning {LLMs} Fail on Unanswerable Questions},
  author    = {Kirichenko, Polina and Ibrahim, Mark and Chaudhuri, Kamalika and Bell, Samuel J.},
  booktitle = {Advances in Neural Information Processing Systems},
  year      = {2025},
  publisher = {Curran Associates, Inc.},
}

@article{brenner2009continual,
  author    = {Brenner, Michael and Nebel, Bernhard},
  title     = {Continual Planning and Acting in Dynamic Multiagent Environments},
  journal   = {Autonomous Agents and Multi-Agent Systems},
  volume    = {19},
  number    = {3},
  pages     = {297--331},
  year      = {2009},
  publisher = {Springer}
}

@inproceedings{kinny1991commitment,
  author    = {Kinny, David N. and Georgeff, Michael P.},
  title     = {Commitment and Effectiveness of Situated Agents},
  booktitle = {Proceedings of the 12th International Joint Conference on Artificial Intelligence (IJCAI)},
  pages     = {82--88},
  year      = {1991},
  publisher = {Morgan Kaufmann}
}

@article{cohen1990intention,
  author    = {Cohen, Philip R. and Levesque, Hector J.},
  title     = {Intention is Choice with Commitment},
  journal   = {Artificial Intelligence},
  volume    = {42},
  number    = {2--3},
  pages     = {213--261},
  year      = {1990}
}

@article{scherl2003knowledge,
  author    = {Scherl, Richard B. and Levesque, Hector J.},
  title     = {Knowledge, Action, and the Frame Problem},
  journal   = {Artificial Intelligence},
  volume    = {144},
  number    = {1--2},
  pages     = {1--39},
  year      = {2003}
}

@inproceedings{katz2024thought,
  author    = {Katz, Michael and Kokel, Harsha and Srinivas, Kavitha and Sohrabi, Shirin},
  title     = {Thought of Search: Planning with Language Models Through The Lens of Efficiency},
  booktitle = {Advances in Neural Information Processing Systems},
  volume    = {37},
  year      = {2024}
}

@inproceedings{guan2023worldmodels,
  author    = {Guan, Lin and Valmeekam, Karthik and Sreedharan, Sarath and Kambhampati, Subbarao},
  title     = {Leveraging Pre-trained Large Language Models to Construct and Utilize World Models for Model-based Task Planning},
  booktitle = {Advances in Neural Information Processing Systems},
  volume    = {36},
  year      = {2023}
}

@misc{liu2023llmp,
  author        = {Liu, Bo and Jiang, Yuqian and Zhang, Xiaohan and Liu, Qiang and Zhang, Shiqi and Biswas, Joydeep and Stone, Peter},
  title         = {{LLM+P}: Empowering Large Language Models with Optimal Planning Proficiency},
  year          = {2023},
  eprint        = {2304.11477},
  archivePrefix = {arXiv}
}

@inproceedings{dagan2025plancraft,
  title     = {Plancraft: An Evaluation Dataset for Planning with {LLM} Agents},
  author    = {Dagan, Gautier and Keller, Frank and Lascarides, Alex},
  booktitle = {Proceedings of the Second Conference on Language Modeling (COLM)},
  year      = {2025},
}

@article{mcdermott2000aips98,
  title   = {The 1998 {AI} Planning Systems Competition},
  author  = {McDermott, Drew M.},
  journal = {AI Magazine},
  volume  = {21},
  number  = {2},
  year    = {2000},
  doi     = {10.1609/aimag.v21i2.1506},
}

\onecolumn
\appendix
\section*{Appendices}
\makeatletter
\def\input@path{{appendix/}}
\makeatother
\graphicspath{{appendix/}}
\lstset{inputpath=appendix/}
\counterwithin{table}{section}
\counterwithin{figure}{section}
\drafttrue

\newcommand{\offlineoutcome}{The plan has 4 actions and is valid; the optimal length is 4.}
\newcommand{\onlineoutcome}{The agent makes 8 calls, 4 of them executed actions, and the environment accepts its goal claim; the committed plan has 4 actions (optimal length 4).}

This document complements the paper. Appendix~\ref{app:domains} lists the five planning domains in PDDL. Appendix~\ref{app:tasks} gives two example tasks from the new domains, with the natural-language briefs that offline planners receive. Appendix~\ref{app:online} shows the online interface, an offline run and an online episode on one of these tasks. Appendix~\ref{app:prompts} gives every prompt: those that generated the briefs, those of the six prompting strategies and those of the three agent architectures. Appendix~\ref{app:tokens} reports token use and cost per model, and Appendix~\ref{app:perdomain} reports results per domain and task. The accompanying archive contains the code and all 50 tasks.

{\small\setcounter{tocdepth}{2}\renewcommand{\contentsname}{\protect\normalsize Contents}\tableofcontents}

\section{Planning domains}
\label{app:domains}

The domain files below are those used by the validator and by the agent environment, with comments removed. Plan quality is plan length in every domain: each action in the four typed domains increases \texttt{total-cost} by one, and \dom{Blocksworld} has no cost function. \dom{Blocksworld} and \dom{Logistics} follow the standard encodings; \dom{Gripper} adds a single condition, that the robot cannot enter a locked room, which only Gripper task~9 uses. In these three domains, object names in the problem files are random identifiers (Section~4 of the paper). The problem files of all 50 tasks are in the archive, under \texttt{examples/<domain>/problems\_pddl/}, together with their briefs (\texttt{examples/<domain>/prompts.json}).

In the listings, PDDL keywords are shown in \textcolor{accent}{\textbf{blue}} and variables in \textcolor{vermilion}{vermilion}.

\subsection{Blocksworld}
Blocks are untyped. A single arm picks up a clear block from the table, unstacks a clear block from another block, puts the block it holds down on the table or stacks it on a clear block; it holds at most one block at a time.
\pddlfile{Blocksworld domain}{gen/pddl_domain_blocks.txt}

\subsection{Gripper}
A robot with two grippers moves balls between rooms. Each gripper holds at most one ball. The robot can move from its room to any room that is not locked.
\pddlfile{Gripper domain}{gen/pddl_domain_gripper.txt}

\subsection{Logistics}
Packages are carried by trucks between locations of the same city and by aeroplanes between airports. Trucks are confined to their city; aeroplanes fly only between airports.
\pddlfile{Logistics domain}{gen/pddl_domain_logistics.txt}

\subsection{LoadBalancing}
Traffic buckets are assigned to servers. A bucket can be assigned only to an online server in its region that runs the bucket's service and has spare capacity; numeric fluents hold the load and the capacity of each server. A server can be powered off only when its load is zero, and services can be started and stopped on online servers. The capacity condition is written as $\mathit{load}+1\le\mathit{capacity}$, which the validator evaluates exactly; for the integer loads that arise, it is equivalent to $\mathit{load}<\mathit{capacity}$.
\pddlfile{LoadBalancing domain}{gen/pddl_domain_balancer.txt}

\subsection{Hospital}
Waiting patients are admitted to free beds in the ward they require, and patients who require isolation need an isolation bed. A bed must also be clean and unblocked to receive a patient. Transfers and discharges leave the vacated bed dirty, and a dirty bed must be cleaned before it is free again. Free beds can be blocked, for example for maintenance, and unblocked.
\pddlfile{Hospital domain}{gen/pddl_domain_hospital.txt}

\section{Example tasks}
\label{app:tasks}

We illustrate the two new domains with their first tasks. For each, we give the problem file (comments removed), an optimal plan, which the validator accepts, and the two briefs that offline planners receive. The briefs were generated from the PDDL files with the prompts of Appendix~\ref{app:briefprompts} and checked by hand.

\subsection{LoadBalancing, task 1}
\label{app:lb1}
\pddlfile{Problem file}{gen/pddl_problem_balancer_1.txt}
\begin{tcolorbox}[appbox, colframe=executed, colback=executed!4, title={Optimal plan (4 actions)}]
\begin{lstlisting}[style=pddl]
(unassign-bucket b2 s2_eu)
(assign-bucket b2 s1_eu eu_west search)
(assign-bucket b3 s1_eu eu_west search)
(power-off s2_eu)
\end{lstlisting}
\end{tcolorbox}
\brieffile{Domain brief}{gen/brief_domain_balancer.txt}
\brieffile{Instance brief}{gen/brief_instance_balancer_1.txt}

\subsection{Hospital, task 1}
\pddlfile{Problem file}{gen/pddl_problem_hospital_1.txt}
\begin{tcolorbox}[appbox, colframe=executed, colback=executed!4, title={Optimal plan (2 actions)}]
\begin{lstlisting}[style=pddl]
(admit p1 b1 w_med)
(admit p2 b2 w_med)
\end{lstlisting}
\end{tcolorbox}
\brieffile{Domain brief}{gen/brief_domain_hospital.txt}
\brieffile{Instance brief}{gen/brief_instance_hospital_1.txt}

\section{The online environment and example runs}
\label{app:online}

\subsection{The online interface}
An online agent receives the instruction in Appendix~\ref{app:agentprompts} and acts through nine tools. Seven are epistemic and leave the state unchanged: \tool{get\_task\_overview}, \tool{list\_objects}, \tool{describe\_object}, \tool{list\_actions}, \tool{describe\_action}, \tool{get\_state} and \tool{get\_history}. The ontic tool \tool{act} executes one grounded action, and the terminal tool \tool{submit} ends the episode, either with the claim that the goal is reached or, as \tool{submit(unsolvable=True)}, with the claim that the task is unsolvable. Every call consumes one unit of the budget (Definition~3.3 of the paper). The ReAct and Reflexion actors also see the \tool{finish} tool that DSPy's ReAct module adds (Appendix~\ref{app:agentprompts}); the harness treats a call to it as a goal claim, like \tool{submit()}.

Below are the observations that the tools return on LoadBalancing task~1. Box colours mark the kind of call:
\begin{center}\legendchip{epistemic}{epistemic}\quad\legendchip{executed}{executed action}\quad\legendchip{illformed}{ill-formed call}\quad\legendchip{fatal}{precondition violation}\quad\legendchip{terminal}{terminal}\end{center}

\subsubsection{At the start of the episode}
\toolcall{epistemic}{get\_task\_overview()}{gen/interface/overview.txt}
\toolcall{epistemic}{list\_objects()}{gen/interface/objects.txt}
\toolcall{epistemic}{list\_actions()}{gen/interface/actions.txt}
\toolcall{epistemic}{describe\_object('s1\_eu')}{gen/interface/object.txt}
\toolcall{epistemic}{get\_state()}{gen/interface/state.txt}

\subsubsection{Acting} The call with the unknown object \texttt{b4} is ill-formed: it is rejected, the state is unchanged and the episode continues. The second call executes an applicable action. A goal claim ends the episode, and the environment reports whether the executed plan reaches the goal; here, after one action, it does not.
\toolcall{illformed}{act('assign-bucket b4 s1\_eu eu\_west search')}{gen/interface/illformed.txt}
\toolcall{executed}{act('assign-bucket b3 s1\_eu eu\_west search')}{gen/interface/executed.txt}
\toolcall{epistemic}{get\_history()}{gen/interface/history.txt}
\toolcall{terminal}{submit()}{gen/interface/submit.txt}

\subsubsection{An inapplicable action} In a new episode on the same task, powering off \texttt{s2\_eu} is inapplicable, because the server still has load~1. The call ends the episode with a precondition violation.
\toolcall{fatal}{act('power-off s2\_eu')}{gen/interface/fatal.txt}

\subsection{An offline run}
\label{app:offline-run}
GPT-4o with direct prompting on LoadBalancing task~1. The prompt is the Direct prompt of Appendix~\ref{app:strategyprompts} with the briefs of Appendix~\ref{app:lb1}. \offlineoutcome
\plainfile{Model response}{gen/example_offline_response.txt}
\pddlfile{Parsed plan}{gen/example_offline_plan.txt}
\plainfile{Validator report (VAL, excerpt)}{gen/example_offline_val.txt}

\subsection{An online episode}
\label{app:online-episode}
GPT-4o with ReAct on LoadBalancing task~1. Each step shows the agent's reasoning (in italics), its tool call and the observation it receives, coloured as in Appendix~C.1. \onlineoutcome
\stepbox{epistemic}{1}{get\_task\_overview}{To begin solving the episode, I need to understand the task at hand. I will start by getting an overview of the task.}{get\_task\_overview()}{gen/episode/obs_01.txt}
\stepbox{epistemic}{2}{list\_actions}{I now understand the task: consolidate traffic onto a single server in eu\_west and decommission the other server. I need to know the available actions to plan the steps for achieving this goal.}{list\_actions()}{gen/episode/obs_02.txt}
\stepbox{epistemic}{3}{get\_state}{I have the list of actions available. To consolidate traffic onto a single server, I need to unassign buckets from one server and reassign them to the other. I should check the current state to see the load on each server and which buckets are assigned where.}{get\_state(max\_facts=200)}{gen/episode/obs_03.txt}
\stepbox{executed}{4}{act}{The current state shows that b1 is assigned to s1\_eu, b2 is assigned to s2\_eu, and b3 is unassigned. Both servers are online and running the search service. To consolidate traffic onto a single server, I can unassign b2 from s2\_eu and assign it to s1\_eu, then assign b3 to s1\_eu. This will allow me to power off s2\_eu.}{act(step\_text='unassign-bucket b2 s2\_eu')}{gen/episode/obs_04.txt}
\stepbox{executed}{5}{act}{I have successfully unassigned b2 from s2\_eu. The next step is to assign b2 to s1\_eu, which will consolidate the traffic onto a single server.}{act(step\_text='assign-bucket b2 s1\_eu eu\_west search')}{gen/episode/obs_05.txt}
\stepbox{executed}{6}{act}{I have successfully assigned b2 to s1\_eu. Now, I need to assign b3 to s1\_eu to complete the consolidation of traffic onto a single server. After that, I can power off s2\_eu.}{act(step\_text='assign-bucket b3 s1\_eu eu\_west search')}{gen/episode/obs_06.txt}
\stepbox{executed}{7}{act}{I have successfully assigned b3 to s1\_eu, consolidating all traffic onto a single server. Now, I can proceed to power off s2\_eu, as it has zero load.}{act(step\_text='power-off s2\_eu')}{gen/episode/obs_07.txt}
\stepbox{terminal}{8}{submit}{I have successfully consolidated all traffic onto a single server (s1\_eu) and powered off the other server (s2\_eu) as required by the task. The goal has been achieved, and I can now submit the solution.}{submit(unsolvable=False)}{gen/episode/obs_08.txt}

\section{Prompts}
\label{app:prompts}

\subsection{Brief generation}
\label{app:briefprompts}
The briefs were generated with GPT-5 (\texttt{gpt-5-2025-08-07}, minimal reasoning effort, temperature~1) in two stages, with the system message ``You are a helpful assistant.'' First, the domain brief is generated from the domain file; then each instance brief is generated from its problem file and the domain brief. The placeholders in double braces are replaced by the file contents and by the generated domain brief.
\plainfile{Prompt for the domain brief}{gen/prompt_domain_prompt_template.txt}
\plainfile{Prompt for an instance brief}{gen/prompt_instance_prompt_template.txt}

\subsection{Prompting strategies}
\label{app:strategyprompts}
Every prompt begins with the same context, which presents the two briefs, and the final call of every strategy ends with the same output requirements. In the prompts below, \textcolor{placeholder}{\texttt{\textbf{<CONTEXT>}}} and \textcolor{placeholder}{\texttt{\textbf{<OUTPUT REQUIREMENTS>}}} stand for these two blocks, \textcolor{placeholder}{\texttt{\textbf{<DOMAIN BRIEF>}}} and \textcolor{placeholder}{\texttt{\textbf{<INSTANCE BRIEF>}}} for the briefs, and other terms in angle brackets for outputs of earlier calls. The final plan is taken from the last fenced block of the final call. Tree of Thoughts keeps a beam of two prefixes; at each of six depths, it asks for three candidate continuations of every prefix in the beam and scores them in a separate call, adding each score to the score of its prefix, and finally completes the best prefix. Self-Refine stops early when the critique is exactly \texttt{NO\_ISSUES} or the revised plan equals the previous one.
\promptfile{Context}{shared by every call}{gen/prompt_offline_context.txt}
\promptfile{Output requirements}{end of every final call}{gen/prompt_offline_output_requirements.txt}

\minihead{Direct}
\promptfile{Final call}{temperature 0 \textperiodcentered{} at most 4{,}096 output tokens}{gen/prompt_offline_baseline_1.txt}
\minihead{Zero-shot CoT}
\promptfile{First call}{temperature 0 \textperiodcentered{} at most 4{,}096 output tokens}{gen/prompt_offline_zero_shot_cot_1.txt}
\promptfile{Final call}{temperature 0 \textperiodcentered{} at most 4{,}096 output tokens}{gen/prompt_offline_zero_shot_cot_2.txt}
\minihead{Plan-and-Solve}
\promptfile{First call}{temperature 0 \textperiodcentered{} at most 1{,}024 output tokens}{gen/prompt_offline_plan_and_solve_1.txt}
\promptfile{Final call}{temperature 0 \textperiodcentered{} at most 4{,}096 output tokens}{gen/prompt_offline_plan_and_solve_2.txt}
\minihead{Step-Back}
\promptfile{First call}{temperature 0 \textperiodcentered{} at most 1{,}024 output tokens}{gen/prompt_offline_step_back_1.txt}
\promptfile{Final call}{temperature 0 \textperiodcentered{} at most 4{,}096 output tokens}{gen/prompt_offline_step_back_2.txt}
\minihead{Self-Refine}
\promptfile{Draft}{temperature 0 \textperiodcentered{} at most 4{,}096 output tokens}{gen/prompt_offline_self_refine_1.txt}
\promptfile{Critique (up to three rounds)}{temperature 0 \textperiodcentered{} at most 1{,}024 output tokens}{gen/prompt_offline_self_refine_2.txt}
\promptfile{Revision}{temperature 0 \textperiodcentered{} at most 4{,}096 output tokens}{gen/prompt_offline_self_refine_3.txt}
\minihead{Tree of Thoughts}
\promptfile{Proposal (per prefix and depth)}{temperature 0.7 \textperiodcentered{} at most 900 output tokens}{gen/prompt_offline_tree_of_thought_1.txt}
\promptfile{Evaluation (per prefix and depth)}{temperature 0 \textperiodcentered{} at most 400 output tokens}{gen/prompt_offline_tree_of_thought_2.txt}
\promptfile{Completion}{temperature 0 \textperiodcentered{} at most 4{,}096 output tokens}{gen/prompt_offline_tree_of_thought_3.txt}

\subsection{Agent architectures}
\label{app:agentprompts}
The agents are DSPy modules. DSPy turns the signature of each module (its instruction and its input and output fields) into a system message, and the current values of the input fields into a user message; field markers such as \textcolor{accent}{\texttt{\textbf{[[ \#\# objective \#\# ]]}}} delimit the fields. Below are the messages of the first call of each module on LoadBalancing task~1, rendered with DSPy~3.0.4. The harness logs the exact messages of every call of every episode (\texttt{lm\_calls.json}). The three architectures give the acting model the same instruction (the numbered strategy in each actor's system message). The Reflexion actor additionally sees the lessons in its memory. The Critic actor, which is not a DSPy ReAct module, receives the list of tools with their arguments in its instruction; DSPy's ReAct module adds the same list automatically for the other two agents.

\minihead{ReAct}
\msgfile{System message}{gen/agent_react_system.txt}
\msgfile{User message of the first call}{gen/agent_react_user.txt}
\msgfile{User message of the second call: the trajectory so far is appended after every step}{gen/agent_react_user_second.txt}

\minihead{Reflexion}
\msgfile{Actor: system message}{gen/agent_reflexion_actor_system.txt}
\msgfile{Actor: user message of the first call}{gen/agent_reflexion_actor_user.txt}
\msgfile{Reflector (called after every ill-formed call): system message}{gen/agent_reflexion_reflector_system.txt}
\msgfile{Reflector: example user message}{gen/agent_reflexion_reflector_user.txt}

\minihead{Critic}
\msgfile{Actor: system message}{gen/agent_critic_actor_system.txt}
\msgfile{Actor: user message of the first call}{gen/agent_critic_actor_user.txt}
\msgfile{Critic (reviews every proposed act or submit call): system message}{gen/agent_critic_critic_system.txt}
\msgfile{Critic: example user message}{gen/agent_critic_critic_user.txt}

\section{Token use and cost per model}
\label{app:tokens}

Table~\ref{tab:tokens} gives the mean number of tokens (input and output) and the mean cost per episode of every configuration, over its 50 episodes. Costs use the list prices of Table~3 of the paper. Table~4 of the paper reports the token column averaged over the four models.

\begin{table}[h]
\centering\small
\caption{Mean tokens and cost per episode, per model and method.}
\label{tab:tokens}
\begin{tabular}{@{}l rrrr c rrrr@{}}
\toprule
& \multicolumn{4}{c}{\textsf{Tokens per episode (thousands)}} & & \multicolumn{4}{c}{\textsf{Cost per episode (US cents)}} \\
\cmidrule(lr){2-5}\cmidrule(l){7-10}
Method & GPT-4o-mini & GPT-4o & Flash-Lite & Flash & & GPT-4o-mini & GPT-4o & Flash-Lite & Flash \\
\midrule
\rowcolor{offlinetint}\multicolumn{10}{@{}l}{\emph{Offline (prompting strategies)}} \\
Direct & 1.4 & 1.4 & 1.5 & 1.6 & & 0.03 & 0.53 & 0.02 & 0.03 \\
Zero-shot CoT & 4.4 & 3.8 & 4.1 & 4.8 & & 0.12 & 1.65 & 0.05 & 0.09 \\
Plan-and-Solve & 2.9 & 2.8 & 2.9 & 3.2 & & 0.06 & 0.97 & 0.03 & 0.05 \\
Step-Back & 3.0 & 2.8 & 3.1 & 3.2 & & 0.07 & 1.00 & 0.04 & 0.05 \\
Self-Refine & 9.6 & 3.6 & 7.5 & 7.9 & & 0.20 & 1.12 & 0.08 & 0.11 \\
Tree of Thoughts & 30.0 & 31.0 & 32.1 & 35.9 & & 0.53 & 9.24 & 0.28 & 0.49 \\
\midrule
\rowcolor{onlinetint}\multicolumn{10}{@{}l}{\emph{Online (agent architectures)}} \\
ReAct & 19.7 & 40.2 & 15.5 & 27.0 & & 0.33 & 11.06 & 0.13 & 0.30 \\
Reflexion & 22.0 & 38.9 & 17.4 & 36.4 & & 0.37 & 10.68 & 0.15 & 0.40 \\
Critic & 44.3 & 39.1 & 44.7 & 52.3 & & 0.77 & 11.89 & 0.38 & 0.60 \\
\bottomrule
\end{tabular}

\end{table}

\section{Results per domain and task}
\label{app:perdomain}

Table~\ref{tab:domains} gives success, executability and detection per domain, pooled over the four models and over the six prompting strategies (offline) or the three agent architectures (online). Figure~\ref{fig:tasks} gives the success rate for each solvable task. As in Table~4 of the paper, success and executability are computed on the solvable tasks and detection on the unsolvable ones.

Online success is lower than offline success in every domain. For individual models, averaged over methods, it is lower in 17 of the 20 combinations of model and domain, and five of the 40 tasks are solved more often online. Difficulty does not follow plan length. \dom{Logistics} has the longest plans but the same offline success as \dom{Blocksworld}, and \dom{Gripper}, whose long plans repeat a simple pattern, is the second easiest domain offline. The new domains are not systematically harder than the classical ones: offline, \dom{LoadBalancing} is the easiest domain and \dom{Hospital} the hardest. No configuration solves \dom{Blocksworld} task~10 or \dom{Hospital} tasks~4 and~7 offline, and ten tasks are never solved online. Offline detections are concentrated on \dom{Gripper} task~1 (13 of the 22) and on \dom{Hospital} (7), and no offline episode recognises an unsolvable \dom{Blocksworld} task.

\begin{table}[h]
\centering\small
\caption{Results per domain, pooled over models and methods. Success and executability are computed on the eight solvable tasks of each domain (192 offline and 96 online episodes per domain) and detection on its two unsolvable tasks (48 and 24 episodes). \emph{Best configuration}: success of the best single configuration (model and method) on the solvable tasks of the domain. $\Lstar$: range of optimal plan lengths.}
\label{tab:domains}
\begin{tabular}{@{}l c rr rr rr rr@{}}
\toprule
& & \multicolumn{2}{c}{\textsf{Success (\%)}} & \multicolumn{2}{c}{\textsf{Best configuration (\%)}} & \multicolumn{2}{c}{\textsf{Executable (\%)}} & \multicolumn{2}{c}{\textsf{Detection (\%)}} \\
\cmidrule(lr){3-4}\cmidrule(lr){5-6}\cmidrule(lr){7-8}\cmidrule(l){9-10}
Domain & $\Lstar$ & Offline & Online & Offline & Online & Offline & Online & Offline & Online \\
\midrule
\dom{Blocksworld} & 4--14 & 30.2 & 7.3 & 62.5 & 37.5 & 33.3 & 17.7 & 0.0 & 8.3 \\
\dom{Gripper} & 5--29 & 46.4 & 37.5 & 100.0 & 100.0 & 47.9 & 38.5 & 27.1 & 0.0 \\
\dom{Logistics} & 9--63 & 30.2 & 14.6 & 100.0 & 50.0 & 41.1 & 32.3 & 2.1 & 0.0 \\
\dom{LoadBalancing} & 3--8 & 55.7 & 34.4 & 87.5 & 87.5 & 67.2 & 56.3 & 2.1 & 12.5 \\
\dom{Hospital} & 2--11 & 27.6 & 11.5 & 50.0 & 25.0 & 28.1 & 21.9 & 14.6 & 20.8 \\
\midrule
All domains & 2--63 & 38.0 & 21.0 & 77.5 & 40.0 & 43.5 & 33.3 & 9.2 & 8.3 \\
\bottomrule
\end{tabular}

\end{table}

\begin{figure}[h]
\centering
\includegraphics[width=\linewidth]{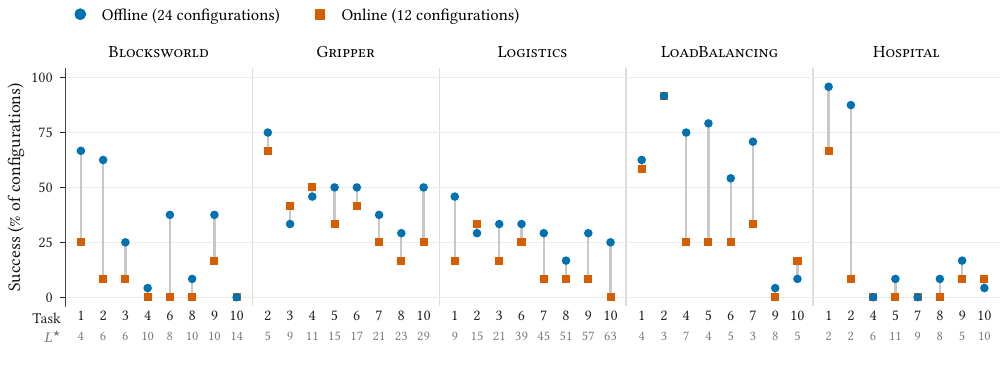}
\caption{Success on each solvable task: the share of the 24 offline configurations (four models and six strategies) and of the 12 online configurations (four models and three architectures) that solve it. The second row of labels gives the optimal plan length $\Lstar$; unsolvable tasks are omitted.}
\label{fig:tasks}
\end{figure}

\section{Use of AI tools}
\label{app:ai}

GPT-5 (\texttt{gpt-5-2025-08-07}) generated the natural-language briefs from the PDDL files, with the prompts of Appendix~\ref{app:briefprompts}; the authors checked the briefs by hand. An AI assistant (Claude Opus 5.5, Anthropic) was used to edit parts of the paper’s text from the authors’ own write-up of this work. The authors reviewed all such output and take full responsibility for the paper and the code.

\end{document}